\documentclass[journal]{IEEEtran}

\ifCLASSINFOpdf
\else
\fi

\usepackage{graphicx}
\usepackage{amsfonts,amssymb,amsmath,amsthm,ulem}
\usepackage{algorithm}
\usepackage{algpseudocode}


\usepackage{ulem}
\usepackage{etoolbox}
\makeatletter
\AfterEndEnvironment{algorithm}{\let\@algcomment\relax}
\AtEndEnvironment{algorithm}{\kern2pt\hrule\relax\vskip3pt\@algcomment}
\let\@algcomment\relax
\newcommand\algcomment[1]{\def\@algcomment{\footnotesize#1}}

\renewcommand\fs@ruled{\def\@fs@cfont{\bfseries}\let\@fs@capt\floatc@ruled
  \def\@fs@pre{\hrule height.8pt depth0pt \kern2pt}%
  \def\@fs@post{}%
  \def\@fs@mid{\kern2pt\hrule\kern2pt}%
  \let\@fs@iftopcapt\iftrue}
\makeatother

\usepackage{upgreek}

\usepackage{tabularx} % Better tables
\usepackage{multirow}
\usepackage{booktabs}

\newcommand{\R}{{\mathbb{R}}}
\newcommand{\E}{{\mathbb{E}}}

\newcommand{\lb}{{\langle}}
\newcommand{\rb}{{\rangle}}
\newcommand{\W}{{\mathbf{W}}}
\newcommand{\Ht}{{\mathbf{H}}}
\newcommand{\Y}{{\mathbf{Y}}}
\newcommand{\V}{{\mathbf{V}}}
\newcommand{\Dd}{{\mathbf{D}}}
\newcommand{\Id}{{\mathbf{I}}}

\newcommand{\yd}{{\mathbf{y}}}

\newcommand{\bPhi}{{\bar{\bm{\Phi}}}}
\newcommand{\aPhi}{{\bm{\Phi}}}

\newcommand{\tT}{\intercal}

\newcommand{\X}{{\mathbf{X}}}
\newcommand{\x}{{\mathbf{x}}}
\newcommand{\h}{{\mathbf{h}}}

\newcommand{\LG}{{\mathbf{G}}}
\newcommand{\LH}{{\mathcal{H}}}
\newcommand{\LE}{{\mathbf{E}}}
\newcommand{\LM}{{\mathbf{M}}}
\newcommand{\dotr}[1]{#1^{\bullet}} 
\newcommand{\haPhi}{\hat{{\bm{\Phi}}}}
\newcommand{\hW}{{\hat{\mathbf{W}}}}
\newcommand{\hHt}{{\hat{\mathbf{H}}}}
\newcommand{\eps}{{\epsilon_0}}
\newcommand{\TS}{{U_S}}
\newcommand{\St}{\mathcal{O}}

\newcommand{\anti}{{ ( \mathrm{anti}  ) }}
\newcommand{\symm}{{ ( \mathrm{sym}  ) }}
\newcommand{\bfGamma}{{ \bm{\Gamma}  }}

\usepackage{xcolor}
\definecolor{darkgreen}{rgb}{0,0.4,0}

\newcommand{\stkout}[1]{\ifmmode\text{\sout{\ensuremath{#1}}}\else\sout{#1}\fi}

\usepackage{bm}

\usepackage{mathtools}

\DeclareMathOperator*{\argmin}{arg\,min}
\DeclareMathOperator*{\D}{Diag}

\newtheorem{lemm}{Lemma}

\newtheorem{thm}{Theorem}
\newtheorem{prop}{Proposition}

\newtheorem{rmq}{Remark}

\usepackage{hyperref}

\usepackage{tikz,pgfplots,pgfplotstable}   
\pgfplotsset{compat=1.7}
\usepgfplotslibrary{groupplots}
\def\domain{\Upsilon}

\begin{document}

\title{Leveraging Joint-Diagonalization in Transform-Learning NMF}
\author{Sixin Zhang, \IEEEmembership{Member, IEEE}, Emmanuel Soubies, \IEEEmembership{Member, IEEE}, and C\'edric F\'evotte, \IEEEmembership{Fellow, IEEE}
\thanks{The authors are with IRIT, Universit\'e de Toulouse, CNRS, Toulouse, France. Email: firstname.lastname@irit.fr. This work was supported by the European Research Council (ERC FACTORY-CoG-6681839), the French National Research Agency (ANR ANITI-3IA) and the National Research Foundation, Prime Minister’s Office, Singapore under its Campus for Research Excellence and Technological Enterprise (CREATE) programme.
}
}


\maketitle
\begin{abstract}
Non-negative matrix factorization with transform learning (TL-NMF) is a recent idea that aims at learning data representations suited to NMF.
In this work, we relate TL-NMF to the classical matrix joint-diagonalization (JD) problem. We show that, when the number of data realizations is sufficiently large, TL-NMF can be replaced by a two-step approach---termed as JD+NMF---that  estimates the transform through JD, prior to NMF computation. In contrast, we found that when the number of data realizations is limited, not only is JD+NMF no longer equivalent to TL-NMF, but the inherent low-rank constraint of TL-NMF turns out to be an essential ingredient to learn meaningful transforms for NMF.     
\end{abstract}

\begin{IEEEkeywords} 
Transform learning, Nonnegative matrix factorization, Joint-diagonalization, Statistical estimation, Nonconvex optimization, Quasi-Newton method.
\end{IEEEkeywords}
\IEEEpeerreviewmaketitle

%\tableofcontents

\section{Introduction}

\IEEEPARstart{L}{et} $\Y \in \R^{M \times N}$ be a random matrix that satisfies the following moment conditions: $\forall (m',m,n)$,
\begin{equation}\label{eq:GenModel}
\left\lbrace 
\begin{array}{ll}
    \E ( [\bPhi \Y]_{mn} ) = 0, & \text{(mean)} \\
    \E ( [\bPhi \Y]_{mn}^2 ) = [ \bar \W \bar \Ht]_{mn}, & \text{(variance)} \\
    \E ( [\bPhi \Y]_{mn}  [\bPhi \Y]_{m'n} ) = 0, & \text{(covariance)} 
\end{array}\right.
\end{equation}
% \ced{bold 0 ?}\sixin{not really, because it is a scalar}
where $\bPhi \in \R^{M \times M}$ is an orthogonal matrix and 
$\bar \W \in \R_+^{M \times \bar K}$, $\bar \Ht \in \R_+^{\bar K \times N}$ are two nonnegative matrices such that $[\bar \W \bar \Ht ]_{mn} > 0$. Here $\bar{K}$ denotes the non-negative rank of $\bar \W \bar \Ht$~\cite{Berman1994}. These fairly general conditions encompass, for instance, the popular Gaussian composite model (GCM)~\cite{Fevotte2009,Smaragdis2014}, which reads as
\begin{equation}\label{eq:GCM}
    [\bPhi \Y]_{mn} \underset{\text{i.i.d.}}{\sim} \mathcal{N}(0,[\bar  \W \bar \Ht]_{mn}), \; \forall (m,n),
\end{equation}
where $\mathcal{N}$ denotes the real-valued Gaussian distribution.

In this paper,  we are concerned with the problem of estimating the triplet $(\bPhi,\bar \W,\bar \Ht)$ from $S$ i.i.d. realizations of~$\Y$, which we denote by $\{\Y^{(s)}\}_{s=1}^S$ (this includes the most frequent scenario $S=1$). In the literature, this problem is referred to as transform-learning nonnegative matrix factorization (TL-NMF)~\cite{Fagot2018,Zhang2020On} and is also a special case of independent low-rank tensor analysis (ILRTA)~\cite{Yoshii18}. When the columns of $\Y$ are composed of overlapping segments of a temporal signal  and $\bPhi$ is a fixed frequency-transform matrix such as the DCT (we only consider real-valued transforms in this paper), the product $\bPhi \Y$ defines a short-time frequency transform. Its power values $[\bPhi \Y]_{mn}^2$ define a \emph{spectrogram}. Equation~\eqref{eq:GenModel} dictates that the spectrogram is given in expectation by the nonnegative matrix factorization (NMF) $\bar \W \bar \Ht$. The zero-mean assumption is valid for centered temporal signals such as audio signals. The zero-covariance assumption is a working assumption that reflects uncorrelation of the transform coefficients. The dictionary matrix $\bar \W$ captures meaningful spectral patterns while the activation matrix $\bar \Ht$ describes how the spectral samples are decomposed onto the dictionary.

NMF has found a large range of applications in audio signal processing, such as source separation or music transcription~\cite{Smaragdis2014,Vincent2018,Makino2018,Gillis2020}. 
In TL-NMF, the fixed transform assumption is relaxed and $\bPhi$ is estimated together with $\bar \W$ and $\bar \Ht$ \cite{Fagot2018}. This was shown to lead in some cases to more meaningful representations. For example, it was shown in \cite{Yoshii18} that  better source separation performance can be achieved with adaptive learned transform. In \cite{Zhang2020On}, it was shown that learning  $\bPhi$ allows to capture harmonic atoms with fundamental frequencies that precisely match the music notes in the data, when a DCT transform can only abide by a fixed frequency grid. TL-NMF was inspired by works about learning sparsifying transforms \cite{Ravishankar2013} and to a lesser extent independent component analysis (ICA) \cite{Hyvarinen2001,Comon2010}.

TL-NMF can be tackled through the resolution of the following optimization problem
\begin{equation}\label{eq:OptimCS}
    (\aPhi^\ast,\W^\ast,\Ht^\ast) \in \argmin_{\substack{\aPhi \in \St(M) \\ (\W,\Ht) \in F_{K}}} C_S(\aPhi,\W,\Ht),
\end{equation}
where $\St(M)$ denotes the real orthogonal matrix group on $\R^M$ and $F_{K} \subset \R^{M \times  K}_{+} \times \R^{K \times N}_{+}$  is a constraint set for the nonnegative matrices $\W$ and $\Ht$ which will be made explicit later. Here $K$ represents the factorization rank whose choice depends on the considered dataset and application~\cite{Gillis2020,Fu2019}.
The objective function $C_S$ in~\eqref{eq:OptimCS} is defined by
\begin{multline}\label{eq:CS}
C_S(\aPhi,\W,\Ht) = 
\sum_{m,n=1}^{M,N}   \frac{ \E_S ( [\aPhi \Y]_{mn}^{2}) + \eps 
}{  [\W \Ht]_{mn} + \eps  }  
 \\ +  \log  \left(  [\W \Ht]_{mn} + \eps \right) 
\end{multline}
where $\E_S ( [\aPhi \Y]_{mn}^{2})$ is a consistent estimator of  $\E ( [\aPhi \Y]_{mn}^{2})$ computed from the $S$ i.i.d. realizations $\{\Y^{(s)}\}_{s=1}^S$, and $\eps \geq 0$. 

%\ced{When only one realization of the signal is available ($S=1$, a most frequent scenario), the empirical estimator $\E_S ( [\aPhi \Y]_{mn}^{2})$ may be replaced by a local moving average under suitable local stationary assumptions. This is further discussed in Section~\ref{sec:conc} towards the end of the paper. In the following, we assume $S$ i.i.d. realizations to be available for simplicity. [à discuter si on laisse cet ajout]}

The rationale behind the choice of Problem~\eqref{eq:OptimCS} is that, when
% \sixin{
% To answer R1, we may say instead: The importance the condition \eqref{eq:GenModel} is that, when
% }
\begin{itemize}
    \item[i)] $S=\infty$ (i.e., $\E_S ( [\aPhi \Y]_{mn}^{2})=\E ( [\aPhi \Y]_{mn}^{2})$), and  $\eps=0$, 
    \item[ii)] $K \geq \bar K$ (i.e., the NMF rank $K$ in~\eqref{eq:OptimCS} is larger than $\bar K$, the true rank in model~\eqref{eq:GenModel}),
%    \item[iii)] for all $(m,n)$,  $ [\bar \W \bar \Ht ]_{mn} > 0$. 
\end{itemize}
each global minimizer $(\aPhi^\ast,\W^\ast,\Ht^\ast)$ of $C_S$ is such that the rows of $\aPhi^\ast$ and $\bPhi$ span the same subspaces (we say that $\bPhi$ is identifiable by solving~\eqref{eq:OptimCS}) and $\W^\ast \Ht^\ast = \bar \W \bar \Ht$~\cite{Zhang2020On}. Moreover, under the GCM, $C_S$ corresponds to the expected negative log-likelihood of $\Y$ conditioned to $(\aPhi,\W,\Ht)$. 

\paragraph*{Contributions and Outline}

This identifiability result %  \sixin{of model~\eqref{eq:GenModel}} 
relies on a key property of $C_S$ that is, in the regime $S=\infty$ and $K \geq \bar K$,
\begin{equation}\label{eq:MarginWH}
 \forall \aPhi \in \St(M), \;    \min_{(\W,\Ht) \in F_{K}} C_S(\aPhi,\W,\Ht) = L_S(\aPhi), 
\end{equation}
where $L_S$ is related to the joint-diagonalization (JD) criterion derived in~\cite{PhamCardoso2001} (and formally defined in Section~\ref{sec:preiminaries}).

This suggests the following two-step 
alternative to Problem~\eqref{eq:OptimCS}
\begin{equation}\label{eq:OptimLS+IS}
\begin{array}{rl}
    \dotr{\aPhi} \in & \displaystyle \mkern-3mu \argmin_{\aPhi \in \St(M)} \, L_S(\aPhi), \\
    (\dotr{\W},\dotr{\Ht}) \in &  \displaystyle  \mkern-15mu  \argmin_{(\W,\Ht) \in F_{K}}  C_S(\dotr{\aPhi},\W,\Ht),
    \end{array}
\end{equation}
which we refer to as JD+NMF. 
Although it is straightforward from~\eqref{eq:MarginWH} that problems~\eqref{eq:OptimCS} and~\eqref{eq:OptimLS+IS} are equivalent in the ideal case  $S=\infty$ and $\eps =0$, their relation in the more practical situation $S< \infty$ and $\eps>0$ (usually required for numerical stability)
needs to be analyzed. In particular: 
\begin{itemize}
    \item How does \eqref{eq:OptimLS+IS} deviate from \eqref{eq:OptimCS} with respect to (w.r.t.) $S$? 
    \item Is one of these two formulations preferable from an optimization perspective? (quality of the reached local minimizer, convergence speed)
\end{itemize}

The main contribution of this work is to provide theoretical and numerical insights to these questions. In Section~\ref{sec:preiminaries}, we set out the basic assumptions and definitions required for our analysis and prove in Theorem~\ref{thm:ExistenceMinCs} that both  TL-NMF (Problem~\eqref{eq:OptimCS}) and JD+NMF (Problem~\eqref{eq:OptimLS+IS}) admit at least one solution. Then, Section~\ref{sec:closeness} is dedicated to the characterization of the closeness between the solution sets of these two problems. We first identify in Theorem~\ref{lem:DistSolSets} quantities that allow to  characterize  three regimes where the solution sets of TL-NMF and JD+NMF are i) disjoint, ii) partially intersecting, and iii)  equal, respectively. 
Then, by analyzing the asymptotic behavior (w.r.t. $S$) of these key quantities, we show that with high probability, 
the gaps between the two solution sets  converge to zero (in terms of the Itakua-Saito divergence \cite{Itakura1968,Fevotte2009}) as $S$ grows 
(Proposition~\ref{prop:Def_QS} and Theorem~\ref{thm:AssympCV}).
Under the GCM, we further prove in Theorem~\ref{thm:pgcmrate} that the rate of convergence is at least of the order of $O(1/S)$. In Section~\ref{sec:algo}, we adapt existing TL-NMF and JD minimization algorithms to Problems~\eqref{eq:OptimCS} and~\eqref{eq:OptimLS+IS} as they come with slightly different objective functions and constraint sets. Finally, in Section~\ref{sec:num}, we deploy these algorithms to illustrate and complement numerically our theoretical findings (tightness of the expected asymptotic behavior, behavior when $S$ is small).

\paragraph*{Notations}
For a matrix  $\X \in \R^{M \times N}$, we denote by $\x_n$, $\underline{\x}_m$ and $[\X]_{mn}$ (or $x_{mn}$) its $n$-th column, $m$-th row, and $(m,n)$-th element, respectively.  Moreover, we use the notation $|\X|^{\circ 2}$ for its element-wise modulus square. 
% For a vector signal $\mathbf{y} \in \R^T$, $y(t)$ \es{[Est-ce qu'on prendrait pas $y[t]$ plutôt pour éviter une confusion avec le $y(t)$ continu de l'intro ?]} \sixin{on a que le vector signal dans le papier, donc ce n'est pas la peine je pense, on peut ecrire $\mathbf{y}(t)$} 
For a vector $\mathbf{y} \in \R^T$, $\mathbf{y} [t]$ denotes its $t$-th element and $\mathbf{y}^\intercal$ denotes its transpose.
We write $\D( \mathbf{y})$ for the diagonal matrix formed out of the vector~$\mathbf{y}$, and $\D( \mathbf{X})$ for the vector formed by the diagonal elements of $\mathbf{X}$. 
%We denote by $\St(M)$ the Stiefel manifold  (real orthogonal matrices) on $\R^M$. \ced{[déjà dit, on le laisse qd même?]}
Finally, $X_n \overset{p}{\rightarrow} X$ stands for the convergence in probability of the sequence $\{X_n\}_n$ of random variables toward the  random variable $X$.

% \es{[TO CHECK: Est-ce qu'il y a des choses pertinentes à ajouter ici ?]}

\section{Preliminaries}\label{sec:preiminaries}

\subsection{Definitions and Useful Results}

% \es{[Est-ce qu'on veut vraiment mettre ``known'' (étant donné que je propose de bouger la section II.C dans III.A pour prendre en compte la remarque du reviewer) ? En fait le Lemme 1 n'a jamais vraiment été explicité ailleurs, ensuite le constraint set et empirical expect ce sont des définitions. Puis la décomposition en (12), bien que triviale, n'a jamais été mise en avant. Ensuite on a des définitions des termes de la décomposition (connu mais c'est juste des définitions). Et enfin, le Lemma 1 est nouveau. Finalement c'est quoi les ``known results'' ? ]}

First, from now on, and throughout the paper, we set $\epsilon_0>0$. Then, let us provide
%an assumption on the model~\eqref{eq:GenModel}, as well as 
a  reformulation of the moment conditions in~\eqref{eq:GenModel} that will be relevant in our analysis.
% \sixin{It is used in \cite{Zhang2020On} to prove the identifiability of the parameter $\bPhi$ in GCM.} \es{[Pas sur que ce soit necessaire de mettre ça. En dehors de l'into on va fixer toujours $S<\infty$ et $\epsilon_0>0$ ce qui ne correspond pas au setting de cette ref.]}

% %\begin{assump}\label{assumpBound}
% %   There exists $\epsilon_1>0$ such that, 
% %   \begin{equation}
%       $\forall (m,n), \; [\bar \W \bar \Ht]_{mn} > 0 , \mbox{ and } \eps > 0 .  $ % \geq \epsilon_1 , \
% %   \end{equation}
% \end{assump}

\begin{prop}\label{prop:jdcond}
The moment conditions~\eqref{eq:GenModel} can be equivalently expressed as
\begin{align}
    \E ( \mathbf{y}_{n} )  &= \mathbf{0}, \label{zerocond}
\\
   % \bPhi \E  ( \mathbf{y}_{n} \mathbf{y}_{n}^\intercal) \bPhi^\intercal &=  
   % \mathrm{Diag}(\bar  \W \bar \h_n), 
   \bPhi \bm{\Sigma}_n \bPhi^\intercal &=  
   \mathrm{Diag}(\bar  \W \bar \h_n), 
\label{jdcond}
\end{align}
   where $\bm{\Sigma}_n = \E  ( \mathbf{y}_{n} \mathbf{y}_{n}^\intercal) $ denotes the covariance matrix of the $n$-th column of $\Y$.
\end{prop}
\begin{proof}
The mean condition in~\eqref{eq:GenModel} implies that for any $n$,  
$ \E (\bPhi \mathbf{y}_{n} ) = \mathbf{0}$, thus we have $\E ( \mathbf{y}_{n} ) = \mathbf{0}$
due to the invertibility of $\bPhi$. The variance
condition implies that the diagonal terms of
$ \bPhi \bm{\Sigma}_n \bPhi^\intercal  $ 
equal to the vector $\bar  \W \bar \h_n$. 
This is because, by definition, 
$[ \bPhi \bm{\Sigma}_n \bPhi^\intercal  ]_{mm} = [\bar \W \bar \Ht]_{mn} = ( \bar  \W \bar \h_n )[m]$. Similarly, the covariance condition 
implies that the off-diagonal terms of $ \bPhi \bm{\Sigma}_n \bPhi^\intercal  $ 
are zero. 
\end{proof}

Finally, to complete the formulation of Problems~\eqref{eq:OptimCS} and~\eqref{eq:OptimLS+IS}
we define below 
the constraint set $F_{K}$ as well as 
the empirical expectation operator  $\E_S (\cdot)$ used in the definition of $C_S$ in~\eqref{eq:OptimCS}.

\paragraph*{Constraint Set}
% Given Assumption~\ref{assumpBound}, 
we define the constraint set $F_{K}$ as
\begin{multline}\label{eq:ConstraintSet}
    F_{K} = \{ (\W,\Ht) \in  \R^{M \times  K}_{+} \times \R^{K \times N}_{+} \; |  \; \forall k, \|\mathbf{w}_k\|_1  = 1 \},
\end{multline}
on which $C_S$ is well-defined (i.e., no singularity) when $\eps>0$. The normalization constraint on each column of $\W$ allows to break the scaling ambiguity between $\W$ and~$\Ht$. 
It is worth mentioning that there are other types ambiguities in NMF~\cite{Fu2018,Gillis2020}. These include permutations of the columns of $\W$ and rows of $\Ht$, or local nonnegativity-preserving rotations. Yet, considering the scaling ambiguity turns out to be an important ingredient to prove the existence of a solution to TL-NMF and JD+NMF (see the proof of Theorem~\ref{thm:ExistenceMinCs} in Appendix~\ref{apdx:Proof_ExistenceMinCs}). This motivates the definition of $F_K$ in~\eqref{eq:ConstraintSet}.

\paragraph*{Empirical Expectation}
 We consider the empirical expectation of $[\aPhi \Y]_{mn}^{2}$, that is
 \begin{equation}\label{eq:EmpiExpect}
     \E_S ( [\aPhi \Y]_{mn}^{2}) = \frac{1}{S} \sum_{s=1}^S [\aPhi \Y^{(s)}]_{mn}^{2}. %  + \epsilon_S
 \end{equation}
%  where $ \epsilon_S>0$. 
%  \sixin{REMOVE is such that $\epsilon_S \to 0$ as $S \rightarrow \infty$}. \es{[Oui enfin, on va mm enlever le $\epsilon_S$ complètement.]} 
 Moreover, we use the notation $\E_S(|\aPhi \Y|^{\circ 2})$ to refer to the point-wise empirical expectation, i.e., $[\E_S(|\aPhi \Y|^{\circ 2})]_{mn}=\E_S([\aPhi \Y]_{mn}^2)$.
% Let $\{x^{s}\}_{s=1}^S$ be $S$ realisations of a random variable $X \in \R$. In this work, we consider the regularized empirical expectation defined as \sixin{this is still a strange notation to me }
%  \begin{equation}\label{eq:EmpiExpect}
%      \E_S ( X) = \frac{1}{S} \sum_{s=1}^S x^s + \epsilon_S,
%  \end{equation}
%   where $\epsilon_S >0$ is such that $\epsilon_S \to 0$ as $S \rightarrow \infty$. Moreover, for a random matrix $\X \in \R^{M \times N}$, we use the notation $\E_S(\X)$ to refer to the point-wise empirical expectation, i.e., $[\E_S(\X)]_{mn}=\E_S([\X]_{mn})$.
% As we shall see in Lemma \ref{lemma:LinkCsLs},
% \ced[This empirical expectation is related to the following 
% empirical covariance matrix]{
Finally, we introduce the following empirical covariance matrix:
\begin{equation}\label{eq:EmpiCov}
\bm{\Sigma}_{n,S}   =    \frac{1}{S} \sum_{s=1}^S \yd_n^{(s)} (\yd_n^{(s)})^\intercal  .
\end{equation}
%  + \eps \mathbf{I}

% The additional diagonal matrix $\eps  \mathbf{I}$ can be interpreted (up to a constant normalization) as a linear shrinkage estimator for covariance matrix estimation \cite{LEDOIT2004365}.  

% regularization which 
% aims to improves the statistical estimation of the convariance matrix~$\bm{\Sigma}_n$. \es{[Il faudrait préciser ce que ça veut dire ``improves the statistical estimation '']}
% \cite{} 
% compared to 
% empirical covariance matrix $\bm{\Sigma}_{n,S} = \frac{1}{S} \sum_{s=1}^S \mathbf{y}_{n}^{(s)} \mathbf{y}_{n}^{(s)\intercal} + \epsilon_S I $ 
% well-conditioned. 

% \es{TODO: Add some words on the role of $\epsilon_S$.}

\subsection{From TL-NMF to JD+NMF}
For any integer $S>0$, one can easily verify
%we have 
the following decomposition of the TL-NMF objective in~\eqref{eq:CS},
\begin{equation} \label{eq:decompCs}
    C_S (\aPhi, \W, \Ht) = L_S (\aPhi) + I_S (\aPhi, \W, \Ht),
\end{equation}
with $L_S$ and $I_S$ defined by
\begin{align}
      & L_S(\aPhi) = MN+ \sum_{m,n=1}^{M,N} \log ( \E_S ( [\aPhi \Y]_{mn}^{2} )  + \eps ) ,
\label{eq:jd} \\
&   I_S(\aPhi,\W,\Ht) = D_\eps (  \E_S ( | \aPhi \Y |^{ \circ 2} )   | \W \Ht    ) , \label{eq:isnmf}
\end{align}
where $D_\eps(\cdot|\cdot)$ denotes % is measuring 
a regularized form of the Itakura-Saito (IS) divergence. 
%originally introduced in the context of short-time speech spectra estimation. 
For two non-negative matrices $\mathbf{A}$ and $\mathbf{B}$, it is defined by
% \begin{equation}\label{eq:IsDivergence}
%     D_\eps(\mathbf{A}|\mathbf{B}) = \sum_{m,n=1}^{M,N} 
%     \left(  \frac{ [\mathbf{A}]_{mn} + \eps } { [\mathbf{B}]_{mn} + \eps }  -  \log \left( 
%     \frac{ [\mathbf{A}]_{mn} + \eps } { [\mathbf{B}]_{mn} + \eps } \right) - 1 \right). 
% \end{equation}

\begin{equation}\label{eq:IsDivergence}
    D_\eps(\mathbf{A}|\mathbf{B}) = \sum_{m,n} 
    \left(  \frac{ [\mathbf{A}]_{mn} + \eps } { [\mathbf{B}]_{mn} + \eps }  -  \log \left( 
    \frac{ [\mathbf{A}]_{mn} + \eps } { [\mathbf{B}]_{mn} + \eps } \right) - 1 \right). 
\end{equation}

% Equations~\eqref{eq:CS} and~\eqref{eq:isnmf} show that when$\aPhi$ is fixed (e.g., to the DCT), the estimation of $\W$ and $\Ht$ is tantamount to the NMF of 

% As known from \cite{Fevotte}, 
% The term $I_S$ is related to Itakura-Saito NMF (IS-NMF)~\cite{Smaragdis2014}. When $\aPhi$ is fixed, its minimization  with respect to $(\W,\Ht) \in F_{K,\epsilon_0}$ leads to an approximate rank $K$ nonnegative factorization of the matrix $ \E_S (| \aPhi \Y |^{\circ 2} )$, i.e.
% \begin{align}
%     \E_S (| \aPhi \Y |^{\circ 2} ) \approx \W \Ht .
%     \label{eq:goalS}
% \end{align}

The term $I_S$ in~\eqref{eq:isnmf} is related to Itakura-Saito NMF (IS-NMF)~\cite{Fevotte2009}. Indeed, when $\aPhi$ is fixed, the minimization of $I_S$ with respect to $(\W,\Ht) \in F_{K}$ produces a NMF of $\E_S (| \aPhi \Y |^{\circ 2} )$ %such that
% \begin{align}
%     \E_S (| \aPhi \Y |^{\circ 2} ) \approx \W \Ht ,
%     \label{eq:goalS}
% \end{align}
where the fit is measured by the IS divergence.

The term $L_S$ is related to joint-diagonalization (JD), as specified by Lemma~\ref{lemma:LinkCsLs}.
\begin{lemm}\label{lemma:LinkCsLs} 
%Let $\eps>0$. 
For all $\aPhi \in \St(M)$, 
\begin{equation}\label{eq:LinkCsLs}
   L_{S} (\aPhi)  = MN + \sum_{n=1}^{N}   \log \det  \D ( \aPhi (  \bm{\Sigma}_{n,S} + \eps \mathbf{I} )
  \aPhi^\intercal ), 
\end{equation}
which corresponds, up to a constant term, to the JD criterion derived in~\cite{PhamCardoso2001}.
\end{lemm}
\begin{proof}
See Appendix~\ref{apdx:Proof_LinkCsLs}.
\end{proof}

The minimization of $L_S$ over $\St(M)$ leads to an orthogonal matrix $\aPhi$ such that $\aPhi \bm{\Sigma}_{n,S} \aPhi^\intercal$ is as diagonal as possible.
% that ``best'' diagonalizes all covariance matrices $\{\bm{\Sigma}_{n,S}\}_n$ in the sense that there exist $\{\vd_n\}_n$ such that
% \begin{equation}
%   \bm{\Sigma}_{n,S} \approx \aPhi^\intercal \D(\vd_n) \aPhi. 
% \end{equation}
The additional diagonal matrix $\eps  \mathbf{I}$ in \eqref{eq:LinkCsLs} can be interpreted (up to a constant normalization) as a linear shrinkage estimator for covariance matrix estimation \cite{Ledoit2004}.  

Hence, from the decomposition~\eqref{eq:decompCs}, TL-NMF can be interpreted as a trade-off between the JD of the  covariance matrices $\{\bm{\Sigma}_{n,S}\}_n$ and the NMF of the variance matrix $ \E_S (| \aPhi \Y |^{\circ 2} )$. In contrast, the alternative scheme~\eqref{eq:OptimLS+IS} attempts to achieve the same goal in a two-step fashion by first solving a JD problem to obtain $\dotr{\aPhi}$ and then solving an IS-NMF problem given $\dotr{\aPhi}$ to get $ (\dotr{\W},\dotr{\Ht})$.%\footnote{\ced{As stated in the introduction, TL-NMF is a special case of ILRTA \cite{Yoshii18}. ILRTA is a more general model that considers the decomposition of positive semi-definite (PSD) matrices onto elementary PSD matrices. The latter matrices may be modeled has having a common joint-diagonalizer. We point out that the JD procedures in the latter case and in ours are unrelated. [J'ai ajouté cette footnote pour éviter les confusions mais on peut l'enlever si ça semble suspect.]}}

%Yoshii18

% This is why it is called JD+NMF. 

\begin{rmq}\label{rmk:eps0}
From Lemma~\ref{lemma:LinkCsLs}, one can see that the ``regularization'' parameter   $\eps>0$
ensures that $L_S$ is continuous over $\St(M)$.
% Indeed, if $\eps=0$ and $S < M$, then $\bm{\Sigma}_{n,S}$ have vanishing eigenvalues. This degeneracy would create zeros in the log-terms of \eqref{eq:LinkCsLs} 
% due to the fact that the diagonal entries of $\aPhi   \bm{\Sigma}_{n,S}  \aPhi^\intercal$ would vanish.
Indeed,  when $S < M$,  $\bm{\Sigma}_{n,S}$  have vanishing eigenvalues meaning that some diagonal entries of  $\aPhi   \bm{\Sigma}_{n,S}  \aPhi^\intercal$ may be  zero. Hence, taking $\eps>0$ ensures that all the log-terms of \eqref{eq:LinkCsLs} are non-degenerate.

\end{rmq}

\begin{rmq}\label{rmk:jd}
The JD criterion derived in \cite{PhamCardoso2001} and essentially given by \eqref{eq:LinkCsLs} does not assume $\aPhi$ to be explicitly orthogonal but merely non-singular. Many orthogonal JD algorithms were designed
in the early age of ICA when whitening was still 
a standard data pre-processing step. 
A seminal example is the Jacobi algorithm 
by Cardoso and Souloumiac \cite{cardoso1996jacobi} for JD with a least-squares criterion.
However, to ensure optimal one-step performance \cite{Car-PerfOrth1994}, 
the whitening step was eventually dropped in the ICA community. Non-orthogonal JD became the mainstream, and many algorithms ensued for various criteria (such as based on least-squares, maximum likelihood or information measures), see, e.g., \cite{pham2001joint,Yeredor2002,Ziehe2004,Absil2006,Souloumiac2009}. More recently, \cite{Bouchard2020} leverages Riemannian optimization to unify many existing methods and introduce new ones under various constraints for~$\aPhi$. 
%However, whether one should use the orthogonal or non-orthogonal JD
%still depends on practical scenarios \cite{Souloumiac2009b}. 

In this work, we restrict $\aPhi$ to be orthogonal because TL-NMF 
aims to generalize short-time frequency transforms for which orthogonality is somehow natural or desired. 
Moreover when $S=\infty$, we showed that the row subspaces of 
$\bar{\aPhi}$ are identifiable \cite{Zhang2020On}. 
Such an analysis is more difficult in the general case (see \cite{Afsari2008,Kleinsteuber2013} for related discussions). Still, removing the orthogonality assumption of $\aPhi$ could be beneficial in practice 
and forms a relevant research direction
(see further discussion in Section~\ref{sec:conc}).
\end{rmq}

\section{Relationship between JD+NMF and TL-NMF}
\label{sec:closeness}

\subsection{Existence of a Solution}

Before analyzing the relation between TL-NMF and~JD+NMF,   it is worth checking that they both admit at least one solution.
To the best of our knowledge, this question has never been addressed in the TL-NMF literature, nor in NMF literature (i.e., existence of global minimizer(s) for $I_S$, see Lemma~\ref{lem:ExistenceMinIs}).
%\sixin{In the literature, these problems have not been studied.}

\begin{thm}\label{thm:ExistenceMinCs}
The solution sets  $\Omega^*$ of TL-NMF (Problem~\eqref{eq:OptimCS}) and $\dotr{\Omega}$ of JD+NMF (Problem~\eqref{eq:OptimLS+IS}) are nonempty and compact.
\end{thm}
\begin{proof}
See Appendix~\ref{apdx:Proof_ExistenceMinCs}.
\end{proof}

\begin{rmq}
The assumption $\eps >0$ is necessary for the functions $L_S$ and $I_S$ to be well-defined and to show the existence of solutions to TL-NMF and JD+NMF. We chose to include $\eps$ in the definitions of $L_S$ and $I_S$ but we could alternatively have set $\eps = 0$ in $L_S$ and $I_S$, and add an inequality of the form $\W \Ht \ge \eps$ in the constraint set $F_K$. This is related to the approach followed by \cite{Takahashi2018} to study the convergence of NMF with a large range of divergences (including the generalized Kullback-Leibler divergence \cite{Hien2021}) under constraints of the form $\W \ge \eps$, $\Ht \ge \eps$.
\end{rmq}

%\ced{[il faut ajouter $\eps >0$ je pense]}
%\es{C'est fait en début de II mais c'était déjà fait aussi en intro mais sans doute pas assez clair. Là en début de II je pense que c'est bon.}

%\ced{OK}

%To simplify the presentation, we adopt the notations $\Omega^*$ and $\dotr{\Omega}$ to refer to the sets of solutions of TL-NMF and JD+NMF, respectively. We next show that the solution set of TL-NMF is compact. 

\subsection{When does JD+NMF meet TL-NMF?}

In Theorem~\ref{lem:DistSolSets}, we characterize the closeness between the solutions of JD+NMF and the solutions of TL-NMF. 
A graphical illustration of this result is depicted in Figure~\ref{fig:scheme}.

% First of all, in order to simplify the presentation, we denote by $\Omega^*$ and $\dotr{\Omega}$ the sets of solutions of TL-NMF and JD+NMF, respectively. Moreover, for any $({\aPhi}^*,{\W}^*,{\Ht}^*) \in \Omega^*$ and $(\dotr{\aPhi},\dotr{\W},\dotr{\Ht}) \in \dotr{\Omega}$, we adopt the following notations
% \begin{align}
%     & \dotr{C_S} = C_S(\dotr{\aPhi},\dotr{\W},\dotr{\Ht}),\\
%   &C_S^* = C_S({\aPhi}^*,{\W}^*,{\Ht}^*),
% \end{align}
% and similarly for $L_S$ and $I_S$. 

\begin{thm}\label{lem:DistSolSets}
Let $\Omega^*$ and $\dotr{\Omega}$ be defined as in Theorem~\ref{thm:ExistenceMinCs}. %the sets of solutions of TL-NMF and JD+NMF, respectively.
Define
\begin{align}
  &  \underline{\lambda}^* = \min_{(\aPhi,\W,\Ht) \in \Omega^*} I_S(\aPhi,\W,\Ht), \label{eq:lamb_under}\\
  &  \bar{\lambda}^* = \max_{(\aPhi,\W,\Ht) \in \Omega^*} I_S(\aPhi,\W,\Ht),  \label{eq:lamb_bar}
\end{align}
and similarly $\dotr{\underline{\lambda}}$ and $\dotr{\bar{\lambda}}$ by replacing $\Omega^*$ by $\dotr{\Omega}$. Then,
\begin{equation}\label{eq:IneqLevSet}
  0 \leq  \underline{\lambda}^* \leq  \bar{\lambda}^* \leq \dotr{\underline{\lambda}} \leq   \dotr{\bar{\lambda}} < +\infty.
\end{equation}
Moreover, 
\begin{align}
    & \dotr{\Omega} \cap \Omega^* \neq  \emptyset  \; \Longleftrightarrow \; \bar{\lambda}^* = \dotr{\underline{\lambda}}, \label{eq:CNS_partialIncl} \\
    & \dotr{\Omega} = \Omega^* \; \Longleftrightarrow \; \underline{\lambda}^* =  \dotr{\bar{\lambda}}. \label{eq:CNS_fullIncl}
\end{align}
\end{thm}
\begin{proof}
See Appendix~\ref{apdx:Proof_DistSolSets}.
\end{proof}

  \begin{figure}
      \centering
      \includegraphics[scale=0.5]{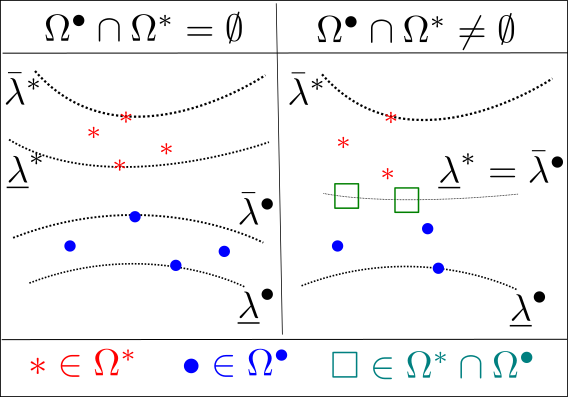}
      \caption{
      Graphical illustration of Theorem~\ref{lem:DistSolSets}. The relationship between the solutions sets $\Omega^*$ and $\dotr{\Omega}$ of TL-NMF and JD+NMF is described in terms of specific level lines of $I_S$. Note that we represented here $\Omega^*$ and $\dotr{\Omega}$ as countable sets for the sake of illustration, but they may not be.}
      \label{fig:scheme}
  \end{figure}

  There are at least two situations where JD+NMF is equivalent to TL-NMF (i.e., \eqref{eq:CNS_fullIncl} is satisfied):
  \begin{itemize}
      \item $S=\infty$ and $K\geq \bar K$,
      \item $S < \infty$ and  $K \geq \min \{M,N\}$.
  \end{itemize}
  Indeed, in these two cases we have that, $\forall \aPhi \in \St(M)$,
  \begin{equation}\label{eq:ExactFacto}
      \exists (\W,\Ht) \in F_{K}, \text{ such that }\; \E_S(|\aPhi \Y|^{\circ 2} ) = \W \Ht,
  \end{equation}
  i.e., there exists an exact factorization of $\E_S(|\aPhi \Y|^{\circ 2} )$. When $K \geq \min \{M,N\}=N$ it suffices to populate $\W$ with $\E_S(|\aPhi \Y|^{\circ 2})$, $\Ht$ with $\Id_N$, and to complete with zeros. A similar exact factorization can be obtained when $K \geq \min \{M,N\}=M$ by switching the roles of $\W$ and $\Ht$.
  %(and similarly when $K \geq \min \{M,N\}=M$)
  For the case $S=\infty$, we refer the reader to the proof of Proposition~\ref{prop:Def_QS} in Appendix~\ref{apdx:proof_Def_QS}.  Consequently, in such cases, we get from~\eqref{eq:ExactFacto} that
  \begin{equation}\label{eq:I_S_eq_0}
      \forall \aPhi \in \St(M), \; \min_{(\W,\Ht) \in F_{K}} I_S(\aPhi,\W,\Ht) = 0.
  \end{equation}
  It then follows from~\eqref{eq:OptimCS}  (resp.,~\eqref{eq:OptimLS+IS}) and~\eqref{eq:decompCs} that
  \begin{equation}
      \forall (\aPhi,\W,\Ht) \in \Omega^* \, (\mathrm{resp.,} \,  \dotr{\Omega}), \; I_S(\aPhi,\W,\Ht) = 0
  \end{equation}
  and thus $ \underline{\lambda}^* =  \dotr{\bar{\lambda}} =0$.  \\
  
  The situation is more involved when $S< \infty$ and $K < \min\{M,N\}$. 
  In this case, we get from Theorem~\ref{lem:DistSolSets} that the closeness between the solutions of JD+NMF and TL-NMF is controlled by the inner $\dotr{\underline{\lambda}} - \bar{\lambda}^*$ and 
  the outer  $\dotr{\bar{\lambda}} - \underline{\lambda}^*$ gaps. Indeed, by definition, we get that for any $({\aPhi}^*,{\W}^*,{\Ht}^*) \in \Omega^*$ and $(\dotr{\aPhi},\dotr{\W},\dotr{\Ht}) \in \dotr{\Omega}$,
   \begin{equation}\label{eq:innerBound}
   0 \leq \dotr{\underline{\lambda}} - \bar{\lambda}^* \leq \dotr{I_S} - I_S^*  \leq \dotr{\bar{\lambda}} - \underline{\lambda}^* ,
 \end{equation}
 where $\displaystyle \dotr{I_S} = I_S (\dotr{\aPhi},\dotr{\W},\dotr{\Ht})$ and $\displaystyle I_S^*  = {I_S}({\aPhi}^*,{\W}^*,{\Ht}^*)$. Moreover, one can see that the outer gap also controls the proximity between minimizers in terms of $C_S$ values because 
 \begin{equation}\label{eq:CSBound}
    0 \leq   \dotr{C_S} - C_S^* \leq \dotr{I_S} - I_S^* 
    \leq  \dotr{\bar{\lambda}} - \underline{\lambda}^*, 
 \end{equation}
 where the first inequality is obtained by combining~\eqref{eq:decompCs} with the fact that $\displaystyle \dotr{L_S} \leq L_S^*$. Note that here, 
 $\displaystyle \dotr{C_S}$ and  $\displaystyle C_S^* $ (resp., $\displaystyle \dotr{L_S}$, and  $\displaystyle L_S^* $)  are the counterparts of $\displaystyle \dotr{I_S}$ and $\displaystyle I_S^*$ for the objectives $C_S$ (resp., $L_S$). In  Section~\ref{sec:assymptotic}, we  shall analyse the asymptotic behaviour of the inner and outer gaps.\\
 
  Finally, let us emphasize that the inequalities~\eqref{eq:IneqLevSet} also reveal that the transforms $\dotr{\aPhi}$ obtained by JD+NMF are 
\textit{at best} as amenable to NMF (in terms of  IS-divergence)  as any $\aPhi^*$
  learned by TL-NMF.

%   Finally, let us emphasize two additional conclusions that can be drawn from Theorem~\ref{lem:DistSolSets}. First, the inequalities~\eqref{eq:IneqLevSet} reveal that the transforms $\dotr{\aPhi}$ obtained by JD+NMF are 
% \textit{at best} as amenable to NMF (in terms of  IS-divergence)  as any $\aPhi^*$   learned by TL-NMF.
%   Second, we see that among the solutions of TL-NMF, JD+NMF may
%   only attain a subset of them.  More precisely, if $ \bar{\lambda}^* = \dotr{\underline{\lambda}}$, JD+NMF attains the solutions of TL-NMF that belong to
% \begin{equation}
%     \left\lbrace (\aPhi,\W,\Ht) \in \Omega^* :   I_S(\aPhi,\W,\Ht) = \bar{\lambda}^* \right\rbrace .
% \end{equation}

\begin{rmq}
The above discussion about the quality of the learned $\aPhi$ in terms of  $I_S$  values may raise the following question: why not directly minimizing $I_S$ with respect to  $(\aPhi,\W,\Ht)$? The latter approach was in particular the one considered in the original TL-NMF paper \cite{Fagot2018} (with an additional sparsity-enforcing term for $\Ht$), an arbitrary choice inherited from IS-NMF \cite{Fevotte2011}. As a matter of fact, in light of our results, several arguments play in favor of minimizing $C_S$ rather than $I_S$. First, under the GCM, $C_S$ comes with a probabilistic interpretation (log-likelihood function). Second, when $S=\infty$,  there exist $(\W,\Ht)$ such that $I_S(\aPhi,\W,\Ht)=0$ for all $\aPhi$ (see~\eqref{eq:I_S_eq_0}). This means that $I_S$ does not constitutes a good measure to discriminate the $\aPhi$. In contrast, the minimization of $C_S$ allows to identify the true transform $\bPhi$ (the term $L_S(\aPhi)$ acts somehow as a penalization term)~\cite{Zhang2020On}. Lastly, this ability to learn meaningful transform (e.g., close to $\bPhi$) by minimizing $C_S$ appears to be also true when $S$ is small (see Section~\ref{sec:expe_S_small}).
\end{rmq}

\subsection{Asymptotic Analysis}\label{sec:assymptotic}
 
 In this section we analyze the closeness between JD+NMF and TL-NMF when   $S<\infty$ and $K \in [\bar K,\min\{M,N\})$.\footnote{The more intricate case $K<\bar{K}$ is not considered in the paper.} % theoretically
 We derive a sufficient condition under which the  outer gap $ \dotr{\bar{\lambda}} - \underline{\lambda}^*$  (and consequently the inner gap $ \dotr{\underline{\lambda}} - \bar{\lambda}^*$) converges (in probability) to zero  as $S$ grows. Given that
 \begin{equation}\label{eq:first_bound_gap}
      \dotr{\underline{\lambda}} - \bar{\lambda}^* 
    \leq  \dotr{\bar{\lambda}} - \underline{\lambda}^* 
    \leq \dotr{\bar{\lambda}},
 \end{equation}
 we first provide in Proposition~\ref{prop:Def_QS} an upper bound of $\dotr{\bar{\lambda}}$ %the outer gap
 that is independent of the solution sets $\Omega^*$ and $\dotr{\Omega}$.
 
\begin{prop}
\label{prop:Def_QS} 
Under condition \eqref{eq:GenModel} and $K \geq \bar K$, 
we have for all $S>0$,
\begin{equation}
    \dotr{\bar{\lambda}}
    \leq \max_{\aPhi \in \St(M)} \, Q_S(\aPhi)
\end{equation}
where
\begin{equation}\label{eq:def_QS}
    Q_S(\aPhi) = D_\eps (  \E_S( | \aPhi \Y |^{  \circ 2} )    \,|\, \E( | \aPhi \Y |^{ \circ 2}) ).
\end{equation}
\end{prop}
\begin{proof}
See Appendix~\ref{apdx:proof_Def_QS}. 
\end{proof}
 
We now derive in Theorem~\ref{thm:AssympCV} a sufficient condition for the upper bound $\max_{\aPhi} Q_S(\aPhi) $  to converge to zero in probability.

\begin{thm}\label{thm:AssympCV}
Assume that the empirical estimator $\E_S ( |\aPhi \Y|^{\circ 2})$ converges uniformly toward $\E( |\aPhi \Y|^{ \circ 2})$ in probability, i.e.,
    \begin{equation}
         \max_{(m,n)} \max_{\aPhi \in \St(M)}\bigg | \E_S ( [\aPhi \Y ]_{mn}^2 )  - \E ( [\aPhi \Y]^{2}_{mn} )  \bigg |  \overset{p}{\to} 0 , 
        \label{erruniformTh}
    \end{equation}
Then % \ced{ $\max_{\aPhi \in \St(M)} Q_S(\aPhi)  \overset{p}{\to} 0$ as  $S \to \infty$.} 
\begin{equation}
    \max_{\aPhi \in \St(M)} Q_S(\aPhi)  \overset{p}{\to} 0 \; \text{ as } \; S \to \infty.
\end{equation}
\end{thm}
\begin{proof}
See Appendix~\ref{sec:AssympCV}.
%Proof_AssympCV
\end{proof}

Under the GCM, we can go one step further than Theorem~\ref{thm:AssympCV} by deriving the convergence rate of $\max_{\aPhi} Q_S(\aPhi)$.

% \sixin{The GCM condition in Theorem 4, is not very precise as we still need 
% mention that the covariance between different (m,n) and (m',n') is zero ...
% }\es{[C'est dans le i.i.d. en dessous du $\sim$ non ?]}

\begin{thm}\label{thm:pgcmrate}
Under  GCM defined by~\eqref{eq:GCM},  condition~\eqref{erruniformTh} is always satisfied.
Moreover, for $t>0$ and $S$ large enough such that 
$h_S = 3 \frac{\sqrt{M}+t}{\sqrt{S}} < 1 $, we have %  + \frac{ \epsilon_S}{ \epsilon_0 } 
\begin{equation}\label{eq:pgcmrate}
     \max_{\aPhi \in \St(M)} Q_S(\aPhi) < MN \frac{h_S^2}{1-h_S}, 
\end{equation}
with probability at least $(1 -2 e^{-t^2/2})^N$.
\end{thm}
\begin{proof}
See Appendix~\ref{sec:AssympCV}. %apdx:Proof_pgcmrate
\end{proof}

It is worth mentioning that the proof of Theorem~\ref{thm:pgcmrate} makes use of existing results on covariance matrix estimation~\cite{vershynin_2012} which can be generalized to the case where the entries of $\Y$ follow some sub-Gaussian distributions~\cite{vershynin_2018}. 

% \sixin{When $\epsilon_S$ is a non-zero constant for all $S$, 
% $h_S$ does not converge to zero, thus $Q_S$ is not necessarily vanishing. 
% }\es{[On devrait plus avoir besoin de s'embetter avec ça quand on aura enlevé $\epsilon_0$ et $\epsilon_S$ et remplacé par un simple $\epsilon$ dans les cost functions]}
% % on $\Y$~\cite{vershynin_2018}.
 
\begin{rmq}\label{rmq:other_rate}
From Theorem~\ref{thm:pgcmrate}, we get that with high probability, as $S$ grows,  $\max_\aPhi Q_S(\aPhi)$ converges at least at a rate of $O( 1/S)$. Indeed, for any $t>0$, there exists $S^*$ such that for all $S \geq S^*$, $h_S < 1$. This means that for $S$ sufficiently large,~\eqref{eq:pgcmrate} holds with probability almost one. Then, the claim comes by observing that, for $S$ large, $h_S = O(1/\sqrt{S}  )$. 
% Finally, we get from this analysis that it makes sense to set $\epsilon_S$ in the order of  $1/\sqrt{S}$ so that $ \max_{\aPhi \in \St(M)} Q_S(\aPhi) =  O(1/S)$, for large $S$. 
\end{rmq}

\begin{rmq}\label{rmq:tightness_rate}
We can provide insights on the tightness of the uniform bound~\eqref{eq:pgcmrate}
by analyzing a point-wise bound such as the asymptotic decay rate of $Q_S(\bPhi)$ with $S$. 
Let us first observe that  
\begin{equation}
   \mkern-15mu  Q_S(\bPhi) = \mkern-10mu \sum_{m,n=1}^{M,N} \mkern-10mu f \left ( X_{mn}^\bPhi  \right ) \text{ for } X_{mn}^\bPhi = \frac{\E_S([ \bPhi \Y ]^{ 2}_{mn}) + \eps  }{\E([ \bPhi \Y ]^{2}_{mn}) + \eps }
\end{equation}
where $f(x)=x-\log(x)-1$. Moreover, from~\eqref{eq:GenModel}, we get that $\E([ \bPhi \Y ]^{ 2}_{mn}) >0$ and, for $S$ sufficiently large, we can assume that $\E_S([ \bPhi \Y ]^{ 2}_{mn}) >0$ (see discussion in Remark~\ref{rmk:eps0}). Hence, we can set (for the purpose of this remark) $\eps =0$. We then get
%Then, by assuming $\eps=0$ (only in this remark), we get 
from the definition of $\E_S$ in~\eqref{eq:EmpiExpect} that  $SX_{mn}^\bPhi$ follows a Chi-squared distribution $\mathcal{X}_S$ of degree $S$. We deduce that
\begin{align}
     \E  ( f \left(X_{mn}^\aPhi \right) )&  = \frac{1}{S}\E (S X_{mn}^\aPhi) + \log(S) - \E \left( \log \left(S X_{mn}^\aPhi \right) \right)-1 \notag \\
     &= \log(S) - \psi(S/2)  - \log(2) \label{eq:proof_Ts_GCM-2}
\end{align}
using the facts that $\E(S X_{mn}^\aPhi) = S$ and $\E( \log (S X_{mn}^\aPhi) ) = \psi(S/2) + \log(2)$ where $\psi$ denotes the di-gamma function.
Then, from the weak law of large number, we  obtain that
\begin{equation}\label{eq:proof_Ts_GCM-3}
  \frac{1}{NM } Q_S(\bPhi)  \overset{p}{\underset{NM \to \infty}{\to}}  (\log(S) - \psi(S/2)  - \log(2) ). 
\end{equation}
Noticing that, when  $S$ is large, $ \psi(S/2) \approx \log(S/2) - 1/S$, we conclude (for $M$, $N$ large and $\eps=0$) that the decay rate of $ Q_S(\bPhi) $ is of the order of $O(1 / S)$ which matches the uniform bound~\eqref{eq:pgcmrate}.
\end{rmq}

\section{Optimization Methods}
\label{sec:algo}

\begin{algorithm}[t]
	\caption{TL-NMF Solver}\label{alg:tl-nmf}
	\textbf{In}: $(J,J_\mathrm{NMF},J_\mathrm{TL}) \in \mathbb{N}^3$, $(\aPhi_0,\W_0,\Ht_0) \in \St(M) \times F_{K}$
	\begin{algorithmic}[1]
		\For{$j= 0 : J-1$}
		\State $(\W_{j+1},\Ht_{j+1}) \gets \textsf{MU}\left(C_S(\aPhi_{j},\cdot,\cdot);J_\mathrm{NMF}, \W_{j}, \Ht_{j}\right)  $ 
		\State $\aPhi_{j+1} \gets \textsf{QN} (C_S(\cdot,\W_{j+1},\Ht_{j+1});J_\mathrm{TL},\aPhi_{j})  $ 
		\EndFor
	\end{algorithmic}
	\textbf{Out}: $\aPhi_J, \W_J, \Ht_J$
		\smallskip
\end{algorithm}
\begin{algorithm}[t]
    \algcomment{Note that we introduced the redundant parameter $J$ in Algorithm~\ref{alg:jd+nmf} so that the input parameters are exactly the same as for the TL-NMF solver in Algorithm~\ref{alg:tl-nmf}. This allows us to ease the computational comparison of the two solvers (see Section~\ref{sec:complexity}).
    The sub-procedures of the form $\textsf{AA} (g;J,\mathbf{X})$ should be read as: executing the $\textsf{AA}$ method on the function $g$ from the initial point $\mathbf{X}$ for $J$ iterations.
    % \textit{$T$ iterations of the  $\textsf{AA}$ method on the function $f$ from the initial point $\mathbf{X}$.
    }
	\caption{JD+NMF Solver} \label{alg:jd+nmf}
	\textbf{In}: $(J,J_\mathrm{NMF},J_\mathrm{TL}) \in \mathbb{N}^3$, $(\aPhi_0,\W_0,\Ht_0) \in \St(M) \times F_{K}$
	\begin{algorithmic}[1]
		\State $\aPhi_J \gets \textsf{QN} (L_S;J \times J_\mathrm{TL},\aPhi_{0})  $ 
		\State $(\W_J,\Ht_J) \gets \textsf{MU}(I_S(\aPhi_{J},\cdot,\cdot);J \times J_\mathrm{NMF}, \W_{0}, \Ht_{0})  $ 
	\end{algorithmic}
	\textbf{Out}: $\aPhi_J, \W_J, \Ht_J$
		\smallskip
\end{algorithm}

To numerically address TL-NMF and JD+NMF problems, we deploy the alternating optimization methods outlined in Algorithms~\ref{alg:tl-nmf} and~\ref{alg:jd+nmf}, respectively. They both rely on the standard  {\em multiplicative updates} (MU) for the NMF factors $\W,\Ht$, and a quasi-Newton (QN) method for the update of the transform $\aPhi$. 
Observing from~\eqref{eq:decompCs} that, given $\aPhi$, the minimization of $C_S(\aPhi,\cdot,\cdot)$ boils down to the minimization of $I_S(\aPhi,\cdot,\cdot)$, the MU steps at line~2 of Algorithms~\ref{alg:tl-nmf} and~\ref{alg:jd+nmf} are the same and are implemented according to~\cite{Fevotte2011}, as summarized in Algorithm~\ref{alg:mm-nmf}. 
The latter is a block-descent algorithm in which the factors are updated in turn. The update of each factor is carried out by one step of {majorization-equalization}, a variant of majorization-minimization \cite{Sun2017} that produces an acceleration while ensuring nonincreasingness of the objective function \cite{Fevotte2011}. This results in the celebrated MU algorithm that is customary in NMF \cite{Gillis2020}.

 Concerning the update of $\aPhi$, we adapt the quasi-Newton (QN) algorithms proposed in~\cite{ablin2018beyond,ablin2019quasi}. These methods were initially developed to solve TL-NMF and JD problems with slightly different objective functions and constraints than those considered in the present work. Numerically, they have been shown to improve the convergence rate of gradient-descent or Jacobi-based methods~\cite{cardoso1996jacobi,wendt2018jacobi,pham2001joint}.

% Algorithm~\ref{alg:mm-nmf} 

% (ME is a variant of majorization-minimization that produces an acceleration while ensuring nonincreasingness of the objective function)

%   \sixin{Put into capture of algorithms?: The sub-procedures of the form $\textsf{AA} (f;T,\mathbf{X})$ should be read as \textit{$T$ iterations of the  $\textsf{AA}$ method on the function $f$ from the initial point $\mathbf{X}$.}
%   }\es{[Oui c'est ce que je voulais faire au début mais je voulais pas le dupliquer dans chaque caption. Peut-être dans la footnote en dessous de Alg 2 si on arrive à garder les deux positionnés l'un sous-l'autre une fois le papier finalisé.]}

%   \sixin{
%   This algorithm was derived for the case where $\W  \Ht $ can have zeros entries. To avoid divisions by zero in the IS-NMF loss, 
%   a modified IS-divergence objective is often considered by adding $\epsilon_0 $ to each entry of $\W \Ht$. }
%   Assumption \ref{assumpBound} is therefore satisfied for $\W \Ht + \epsilon_0 \mathbf{1}_{M \times N} $, which are used on line 3 of Algorithm~\ref{alg:mm-nmf}. 
%   }
%   \es{[Ajouter phrase sur la gestion du $\epsilon_0$ ligne 3 (ne sera sans doute plus utile une fois les modif avec $\epsilon$ faites) + la gestion de la contrainte de normalization sur $\W$.]}
  
\begin{algorithm}[t]
	\caption{Multiplicative updates to minimize $I_S(\aPhi,\cdot,\cdot)$~\cite{Fevotte2011}} \label{alg:mm-nmf}
%	\caption{Itakura-Saito NMF (minimization of $I_S(\aPhi,\cdot,\cdot)$)~\cite{Fevotte2011}} \label{alg:mm-nmf}
	\textbf{In}: $J_\mathrm{NMF} \in \mathbb{N}$, $(\W_0,\Ht_0) \in F_{K}$,  $\V  =  \E_S (| \aPhi \Y |^{\circ 2} )$ %  + \epsilon_S \mathbf{1}_{M \times N} $
	\begin{algorithmic}[1]
		\For{$j= 0 : J_\mathrm{NMF}-1$}
		\State $\hat{\V}  \gets  \W_j \Ht_j + \eps \mathbf{1}_{M \times N} $
		\State $\displaystyle \Ht_{j+1} \gets   \Ht_j \circ \left[   \frac{ \W^\intercal_j  ( ( \hat{\V} )^{ \circ -2 } \circ  \V  )} { \W^\intercal_j ( \hat{\V} )^{ \circ -1 }  }   \right] $
		\State $\hat{\V}  \gets  \W_j \Ht_{j+1} + \eps  \mathbf{1}_{M \times N} $
		\State $\displaystyle \W_{j+1} \gets   \W_j \circ \left[   \frac{ ( ( \hat{\V} )^{ \circ -2 } \circ  \V  ) \Ht^\intercal_{j+1}  } { ( \hat{\V} )^{ \circ -1 } \Ht^\intercal_{j+1}   }  \right] $ 
%		\State Normalize columns of  $\W_j$, and rows of $\Ht_{j}$ accordingly.
		\State Normalize $\W_{j+1}$, $\Ht_{j+1}$ to remove scale ambiguity.
		\EndFor
	\end{algorithmic}
	\textbf{Out}: $\W_{J_\mathrm{NMF}}, \Ht_{J_\mathrm{NMF}}$
	\smallskip
\end{algorithm}  
  
 % 		\Procedure{MM-NMF}{$\{ \Y^{(s)} \},\aPhi,\W,\Ht$}

%		\State \textbf{return} $\W,\Ht$ 

  \subsection{General Principle of QN Methods over $\St(M)$}
  
  For both problems, the main difficulty in deriving a QN method for updating $\aPhi$ comes from the handling of the orthogonality constraint (i.e., $\aPhi \in \St(M)$). In this work, we follow the standard approach that consists in defining a local parameterization $\rho_{\aPhi_j} : \domain \subset \R^{M \times M} \rightarrow \St(M)$ of the neighborhood of an iterate $\aPhi_{j} $ (see for instance~\cite{manton2002optimization}). The idea is then to compute a QN direction based on a quadratic approximation of the local objective function $g \circ \rho_{\aPhi_j}$ (here $g$ stands for either $C_S(\cdot,\W,\Ht)$ or $L_S$). From the second-order Taylor expansion around $\mathbf{0}$, we get
  \begin{equation}\label{eq:TaylorExp}
      g(\rho_{\aPhi_j}(\LE)) = g(\aPhi_j) + \lb \LG , \LE \rb + \frac{1}{2} \lb \LE | \LH  | \LE \rb + O\left( \| \LE \|^3 \right) , 
  \end{equation}
      using the fact that $\rho_{\aPhi_j}(\mathbf{0})= \aPhi_j$, and denoting respectively by $\LG \in \R^{M \times M}$ and $\LH \in \R^{M \times M \times M \times M}$  the gradient matrix and the Hessian tensor of $g \circ \rho_{\aPhi_j}$ at $\mathbf{0}$. In~\eqref{eq:TaylorExp}, the inner products should be read as  $\lb \LG, \LE \rb = \sum_{a,b} [\LG]_{ab} [ \LE]_{ab}$ and $\lb \LE | \LH | \LE \rb = \sum_{a,b,c,d} [\LH]_{abcd} [ \LE ]_{ab} [ \LE ]_{cd} $. 
  
  Because the Hessian is usually costly to compute, we will consider an approximation of it, denoted $\tilde \LH$, and define a QN direction through the resolution of 
  \begin{equation}
\LE_j = \argmin_{\LE \in \domain} \lb \LG , \LE \rb +  \frac{1}{2} \lb \LE | \tilde{ \LH }    | \LE \rb.
\label{eq:QN_dir}
\end{equation}
    Given this QN direction, the current estimate $\aPhi_j$ is updated according to
  \begin{equation}\label{eq:QN_update}
      \aPhi_{j+1} = \rho_{\aPhi_j}(\eta \LE_j),
  \end{equation}
  where $\eta>0$ is obtained by line search (e.g., backtracking~\cite{ablin2018beyond} or Wolfe~\cite{Wolfe1969}). 
  This generic QN approach is summarized in Algorithm~\ref{alg:QN}. In the next two sections, we provide details on its instantiation  to our problems. More precisely, the main task is to compute the gradient $\LG$, the Hessian approximation $\tilde{\LH}$, as well as the QN direction in~\eqref{eq:QN_dir}.

\begin{algorithm}[t]
	\caption{Generic quasi-Newton method to minimize a function $g$ over $\St(M)$} \label{alg:QN}
 	\textbf{In}: $J_\mathrm{QN} \in \mathbb{N}$, $\aPhi_0 \in \St(M)$, $\rho_{\aPhi}(\cdot)$
    % \textbf{In}: $J_\mathrm{TL} \in \mathbb{N}$, $\aPhi_0 \in \St(M)$, $\rho_{\aPhi}(\cdot)$ 
	\begin{algorithmic}[1]
		\For{$j= 0 : J_\mathrm{QN}-1$}
		\State Compute the gradient $\LG$ and the Hessian  approximation $\tilde \LH$ of the local function $g \circ \rho_{\aPhi_{j}}$
		\State $ \LE_j \gets \argmin_{\LE \in \domain} \lb \LG , \LE \rb +  \frac{1}{2} \lb \LE | \tilde{ \LH }    | \LE \rb$ 
		\State Compute $\eta >0$ via line search 
		\State $\aPhi_{j+1} \gets \rho_{\aPhi_{j}}(\eta \LE_j) $  
		\EndFor
	\end{algorithmic}
	\textbf{Out}: $\aPhi_{J_\mathrm{TL}}$
		\smallskip
\end{algorithm}

\begin{rmq}\label{rmq:QN_solvers}
In the following, we adopt two different strategies to parametrize the neighborhood of an iterate $\aPhi_j$. The reason is that each one has been motivated independently by previous works on TL-NMF~\cite{ablin2019quasi} and JD~\cite{ablin2018beyond}. Yet, let us emphasize that for both problems, a QN method could be derived with each of the two local parameterizations described below. 
\end{rmq}

  \subsection{QN Method for TL-NMF}
  
  As in~\cite{ablin2019quasi}, we consider the local parameterization 
  \begin{equation}
  \rho_\aPhi (\LE) = \exp(\LE)\aPhi, \; \forall   \LE \in \domain,  
  \end{equation}
  where   $\domain := \left\lbrace \LE \in \R^{M \times M} : \LE^\tT = -\LE\right\rbrace$
  is the set of anti-symmetric matrices (leading to $\exp(\LE) \in \St(M) $). Then, denoting $\X^{(s)}= \aPhi \Y^{(s)}$, we follow the same steps as in~\cite{ablin2019quasi} to obtain the gradient of $C_S(\cdot,\W,\Ht) \circ \rho_{\aPhi}$ at $\mathbf{0}$,
  \begin{equation}
    [\LG]_{ab} = \frac{2}{S} \sum_{n,s=1}^{N,S} \frac{  [\X^{(s)}]_{an} [\X^{(s)}]_{bn} } { [ \W \Ht ]_{an} + \eps }  ,
    \label{eq:grad_CS}
\end{equation}
as well as the following diagonal Hessian approximation
 \begin{equation}
 [\tilde{\LH}]_{abcd} =   \delta_{ac}  \delta_{bd} [\bm{\Gamma}]_{ab} 
    \label{eq:approx_Hess_CS}
\end{equation}
where 
\begin{equation}\label{eq:Gamma_CS}
     [\bm{\Gamma}]_{ab}=  \frac{2}{S} \sum_{n,s=1}^{N,S} \frac{  [\X^{(s)}]_{b n}^2  } { [ \W \Ht ]_{ an}  + \eps }  .
\end{equation}

In Proposition~\ref{prop:QN_dir_TLNMF}, we provide the closed-form expression of the solution of~\eqref{eq:QN_dir} for such a diagonal $\tilde{\LH}$. Again, this result is an extension of~\cite{ablin2019quasi} to our setting (i.e., minimization of $C_S(\cdot, \W,\Ht)$ instead of $I_S(\cdot, \W,\Ht)$ in~\cite{ablin2019quasi}).

\begin{prop}\label{prop:QN_dir_TLNMF}
For $\tilde{ \LH }$ in~\eqref{eq:approx_Hess_CS}, 
 a solution of Problem \eqref{eq:QN_dir} is given by $\LE \in \domain$ such that
 \begin{equation}\label{eq:prop_QN_dir_TL}
 [\LE]_{ab} = \left\lbrace
 \begin{array}{ll}
     \displaystyle  - \frac{ [\LG^{(\mathrm{anti})}]_{ab}} { [\bm{\Gamma}^{(\mathrm{sym})}]_{ab} } & \text{ if } [\bm{\Gamma}^{(\mathrm{sym})}]_{ab} \neq 0  \\
     0  & \text{ if } [\bm{\Gamma}^{(\mathrm{sym})}]_{ab} =0
 \end{array}\right.
 \end{equation}
where  $\LG^{(\mathrm{anti})} = ( \LG - \LG^\tT)/2$ and $\bm{\Gamma}^{(\mathrm{sym})} =  ( \bm{\Gamma}  + \bm{\Gamma}^\tT ) /2 $. 
Moreover, it is a descent direction, i.e., $\LG^{(\mathrm{anti})} \neq \mathbf{0} $ implies that $\lb \LG , \LE \rb  < 0$. %\ced{[sym plutot que symm?]}
\end{prop}

\begin{proof}
See Appendix~\ref{apdx:proof_prop_QN_dir_TLNMF}.
\end{proof}
  
  \subsection{QN Method for JD}
  
  For the minimization of $L_S$, we rely on the local parameterization discussed in~\cite{manton2002optimization}, that is
   \begin{equation}
  \rho_\aPhi (\LE) = \pi(\aPhi + \LE \aPhi), \; \forall   \LE \in \domain,  
  \end{equation}
  where $\pi$ stands for the projection operator on $\St(M)$ and $\domain$ is again the set of anti-symmetric matrices. The gradient of  $L_S \circ \rho_{\aPhi}$ at $\mathbf{0}$ is then given by
  \begin{equation}
	[\LG]_{ab} = \frac{1}{N} \sum_{n=1}^N \left( 
\frac{ [ \aPhi (\bm{\Sigma}_{n,S} + \eps \Id) \aPhi^\intercal ]_{ab}  }{ [ \aPhi   (\bm{\Sigma}_{n,S} + \eps \Id)  \aPhi^\intercal ]_{aa} } - \delta_{ab}  \right) .
	\label{eq:grad_LS}
\end{equation}
Similarly to~\cite{ablin2018beyond} (where $L_S$ is optimized over the set of \textit{invertible} matrices), we consider the  Hessian approximation
\begin{equation}
	[\tilde{\LH} ]_{abcd} = \delta_{ac}\delta_{bd} 
	[\bfGamma]_{ab} + \delta_{ad} \delta_{bc} - 2 \delta_{abcd} 
\label{eq:approx_Hess_LS}
\end{equation}
where
\begin{equation}
	[\bfGamma ]_{ab} = \frac{1}{N} \sum_{n=1}^N 
	\frac{  [ \aPhi  (\bm{\Sigma}_{n,S} + \eps \Id)  \aPhi^\intercal ]_{bb}  }{ [ \aPhi  (\bm{\Sigma}_{n,S} + \eps \Id)  \aPhi^\intercal ]_{aa} }  .
	\label{eq:Gamma}
\end{equation} 
It is noteworthy to mention that each element of $\LG$ and $\bm{\Gamma}$ is well-defined
as $\forall a \leq M , \aPhi \in \St(M)$,  $[\aPhi  (\bm{\Sigma}_{n,S} + \eps \Id)  \aPhi^\intercal ]_{aa} \geq \eps > 0$. 

Because the Hessian approximation in~\eqref{eq:approx_Hess_LS} is not diagonal, we cannot use Proposition~\ref{prop:QN_dir_TLNMF} to get the QN direction. Yet, the specific structure of $\tilde{\LH}$ in~\eqref{eq:approx_Hess_LS} allows to express the solution of~\eqref{eq:QN_dir}
in closed form, as stated in Proposition~\ref{prop:QN_dir_JD}.

\begin{prop}\label{prop:QN_dir_JD}
For $\tilde{ \LH }$ in~\eqref{eq:approx_Hess_LS}, 
 a solution of Problem \eqref{eq:QN_dir} is given by $\LE \in \domain$  such that
\begin{equation}\label{eq:prop_QN_dir_JD}
[\LE]_{ab} = \left\lbrace
\begin{array}{ll}
    \displaystyle  - \frac{ [\LG^{(\mathrm{anti})}]_{ab}} { [\bfGamma^{(\mathrm{sym})}]_{ab} - 1 } & \text{ if } [\bfGamma^{\symm}]_{ab}  \neq  1  \\
    0  & \text{ if } [\bfGamma^{\symm}]_{ab}  =  1
\end{array}\right.
\end{equation}
where  $\LG^{(\mathrm{anti})} = ( \LG - \LG^\tT)/2$ and $\bfGamma^{(\mathrm{sym})} =  ( \bfGamma  + \bfGamma^\tT ) /2 $. 
Moreover, it is a descent direction, i.e., $\LG^{(\mathrm{anti})} \neq \mathbf{0} $ implies that $\lb \LG , \LE \rb  < 0$. \end{prop}
\begin{proof}
See Appendix \ref{apdx:proof_prop_QN_dir_JD}.
\end{proof}

As opposed to~\cite{ablin2018beyond}, here $\LE$ is constrained to be anti-symmetric. This is to ensure that, for $\aPhi \in \St(M)$,  $\LE \aPhi$ belongs to the tangent space of $\St(M)$ at $\aPhi$. This makes $\tilde{\LH}$ invertible, discarding the need to rely on a pseudo-inverse.
  
  \subsection{Computational Complexity}\label{sec:comp_cmplx}
  
  We now briefly discuss the computational complexity of the two solvers in Algorithm~\ref{alg:tl-nmf} and~\ref{alg:jd+nmf}. First of all, one can see from Algorithm~\ref{alg:mm-nmf} that an iteration of MU has a complexity of the order of $O(KMN)$ (the main cost being the products of matrices with sizes  $M \times K$ and $K\times N$).
  
  Concerning the QN step, the main cost comes from the computation of the gradient and the Hessian approximation  which lead to a complexity per iteration of the order of $O(SNM^2)$ for TL-NMF (see~\eqref{eq:grad_CS} and~\eqref{eq:approx_Hess_CS}) and $O(NM^3)$ for JD (see~\eqref{eq:grad_LS} and~\eqref{eq:approx_Hess_LS}).
  
  The overall complexities of the two solvers are thus of the order of
  \begin{itemize}
      \item $O\left(J(J_\mathrm{NMF}KMN + J_\mathrm{TL}SNM^2)\right)$ for TL-NMF,
      \item $O\left(J(J_\mathrm{NMF}KMN + J_\mathrm{TL}NM^3)\right)$ for JD+NMF.
  \end{itemize}
    In complement of Remark~\ref{rmq:QN_solvers}, we see that the QN strategy deployed for TL-NMF is preferable when $S \leq M$, whereas the one used for JD when $S\geq M$. Yet, for the latter, the storage of the  $N$ covariance matrices $\{ \bm{\Sigma}_{n,S} \}_{n \leq N}$ is required, which results in a memory overload of  $O(N M^2)$.

\section{Numerical Experiments}\label{sec:num}

% \sixin{je vais relire cette partie.} \es{Pas encore fini de mon côté.}

From the theoretical analysis conducted in Section~\ref{sec:closeness}, we know that the solutions of TL-NMF and JD+NMF are getting closer as $S$ grows. More precisely, the closeness between the solution sets of these two problems is controlled by the inner and outer $\lambda$'s gaps (namely $ \dotr{\underline{\lambda}} - \bar{\lambda}^*$ and $ \dotr{\bar{\lambda}} - \underline{\lambda}^*$) which converge to zero with high probability as $S$ grows ({Proposition~\ref{prop:Def_QS} and} Theorem~\ref{thm:AssympCV}). Moreover, under the GCM, they converge at least at a rate of $O(1/S)$ (Theorem~\ref{thm:pgcmrate}). %\ced{[Rqe: Theorem 3 ne va pas jusque dire que les gaps convergent vers 0, il faut aussi la Prop. 2]}\es{[Indeed, done!]}

Yet, these theoretical results only provide a partial answer to the questions we identified in the introduction. For instance, the situation where $S$ is small (including the extreme but very relevant case $S=1$) %or $K < \bar{K}$  % and $K \in [\bar K , \min \{M,N\})$
remains unclear. If TL-NMF and JD+NMF are not equivalent, does one of them provide ``better'' solutions 
%(e.g., \sixin{in terms of $(\aPhi,\W,\Ht)$} \es{[cette précision n'est pas utile. Attention pb de parenthèse avec la suite.]})
% \es{[c'est vrai aussi pour $\W$ et $\Ht$, cf l'ex des Notes]} 
(from an application-oriented perspective) than the other? On the other hand, when they are  equivalent (i.e., when the gap is approaching zero), 
is one of them preferable from the point of view of computational complexity? And lastly, how tight are the derived theoretical bounds? The goal of this section is to provide numerical insights on these questions so as to illustrate and complement our theoretical findings.
The results presented in this section can be reproduced using the 
Python code available at \url{https://github.com/sixin-zh/tlnmf-tsp}.

%\sixin{We shall use the QN methods developed in Section \ref{sec:algo} to answer the above questions}.
%\sixin{We start by formalizing the main questions and their numerical setup.
%We then present and discuss the numerical results on two synthetic examples, one from GCM, the other from an audio example composed of two notes.
%}

\subsection{Datasets}

We consider two simulated datasets obtained respectively from a pure GCM and a blend of two synthetic music notes. The latter departs from GCM while roughly satisfying the conditions~\eqref{eq:GenModel}.

\subsubsection{GCM} In this dataset, the dimensions are set as $M=10$, $N=50$, and $\bar{K} = 5$. The ground-truth transform  $\bPhi$ is fixed to the type-II DCT. The entries of the ground-truth NMF factors $\bar{\W}$ and $\bar{\Ht}$ are independently drawn from a Gamma distribution
\begin{align}
    & [\bar \W]_{mk} \sim \text{Gamma}(a,\theta), \\
    &[\bar \Ht]_{kn} \sim \text{Gamma}(a,\theta),
\end{align}
with shape parameter $a=1$ and scale parameter $\theta=2$. The realizations $\{\Y^{(s)}\}_{s=1}^S$ are generated such that \eqref{eq:GCM} holds. Note that, in the following experiments, when $S$ is changed the $S$ realizations are re-sampled independently using the same $\bar \W$ and $\bar \Ht$.

\subsubsection{Synthetic Music Notes} We follow the construction described in \cite[Section III-A]{Zhang2020On} which we recall for completeness.
   For all $s \in \{1,\ldots,S\}$, we define the signal $\yd^{(s)} \in \R^T$ as
   \begin{equation}
       \yd^{(s)}[t] = \sum_{i=1}^{I=2} \sum_{r=1}^{R=2} 0.5^r \cos\left( r \left( 2 \pi \frac{f_i}{f_0} t +  \theta_i^{(s)} \right) \right) \mathbf{g}_i [t] ,
   \end{equation}
where $ \theta_i^{(s)} \in [0,2\pi)$ is a uniform random phase, $f_i >0$ is a fundamental frequency, $f_0 >0$ denotes the sampling frequency, $r$ is an integer number,  and $\mathbf{g}_i \in \R^T$ is a positive envelop that varies slowly over $t \in \{ 1,\ldots,T\}$. In other words, $\yd^{(s)}$ is the sum of $I=2$ 
pure music notes with frequencies $f_1$ and $f_2$ that each contain $R=2$ harmonics. 
As in \cite{Zhang2020On}, we set the fundamental frequencies to $f_1=440$Hz (corresponding to the note A4) and $f_2=466.16$Hz (A4\#), and fix $T=15000$ and $f_0 = 5000$Hz. This leads to a signal of duration 3s. The envelops $\mathbf{g}_1$ and $\mathbf{g}_2$ are such that the two notes are played separately in the first two seconds and then simultaneously in the third second. Finally, the random nature of the phases $ \theta_i^{(s)}$ allows us to generate $S$ independent signals $\{\yd^{(s)}\}_{s=1}^S$. 
% \ced[From these signals, we then apply a short-time window transform with a Turkey window of duration 40ms to generate the matrices $\{\Y^{(s)}\}_{s=1}^S$, each being of size  $M=200$ by $N=151$.]{Each signal is then segmented into 40ms windowed frames (with Tukey window and 50\% overlap) to generate the matrices $\{\Y^{(s)}\}_{s=1}^S$, each being of size  $M=200$ by $N=151$. [Sixin, il y a bien de l'overlap ?]}

%From these signals, we then apply a short-time window transform with a Turkey window of duration 40ms to generate the matrices $\{\Y^{(s)}\}_{s=1}^S$, each being of size  $M=200$ by $N=151$. 

As analyzed in \cite[Section III-B]{Zhang2020On}, the underlying random matrix $\Y$ 
roughly follows the conditions \eqref{eq:GenModel}. A subset of the atoms of the ground-truth matrix $\bPhi$ captures the fundamental frequencies $\{f_1,f_2\}$ and their harmonics $\{ 2f_1, 2f_2\}$ (two orthogonal harmonic atoms per frequency to account for phase shifts). The ground-truth NMF factors $\bar \W$ and $\bar \Ht$ capture the polyphonic spectra and the activations of the two notes (hence $\bar{K}=2$).

\begin{table}[t]
\normalsize
 \caption{\label{tab:params}Parameters used in the experiments: $K$ is the common dimension of the NMF factors ($\W$,$\Ht$); $\epsilon_0$  is defined in~\eqref{eq:CS}; $J$, $J_\mathrm{TL}$, and $J_\mathrm{NMF}$ are the numbers of iterations in Algorithms~\ref{alg:tl-nmf} and~\ref{alg:jd+nmf}; $P$ denotes the number of random  initializations (see Appendix~\ref{apdx:init_strategy}). }
    \centering
    \begin{tabular}{ccc}
        \toprule
        Dataset & GCM & Music Notes  \\
       \midrule \midrule
         $K$ & $5$ ($= \bar K$) & $2$ (see~\cite{Zhang2020On})  \\
        %  $\epsilon_S$ & $10^{-8}$ & $5\times 10^{-7}$\\
         $\eps$ & $10^{-8}$ & $5\times 10^{-7}$ \\
         $J_\mathrm{TL}$ & $1$ & $1$ \\
         $J_\mathrm{NMF}$ & $10$ & $10$ \\
         $J$ & $1000$ & $ 100$ \\
         $P$ & 100 & $10$  \\  % \sixin{Try 1000?}
        \bottomrule
    \end{tabular}
\end{table}

\subsection{Parameter Settings}

The algorithm and model parameter values for all experiments are set to those listed in Table~\ref{tab:params}.
Let us comment these choices. 
For the GCM dataset, we set $K=\bar{K}$, and $\eps \ll \min_{m,n} [ \bar \W \bar \Ht ]_{mn} \simeq 1.57$. 
It respects the conditions of our theoretical results in Section \ref{sec:closeness}. 
For the music notes dataset, we set $K=2$ to learn a NMF that can separate the two music notes. Regarding $\eps$, we fixed it manually, as a trade-off between numerical stability and quality of the solution in terms of recovering the fundamental frequencies of the music notes.

The number of iterations $J_\mathrm{TL}$ and $J_\mathrm{NMF}$ have been selected so as to maximize the convergence speed of TL-NMF, as they have no effect on the convergence of JD+NMF due to its two-step nature. In particular, we observed that TL-NMF  enjoys a faster convergence when more NMF than TL updates are performed per outer iteration. 
 
Finally, because they address non-convex problems, we ran the algorithms from $P$ different initializations and we retained the best solution (in terms of objective function) for each problem. As such, we assume that we reached points that belong to (or at least are close to) the solution sets $\Omega^*$ and $\dotr{\Omega}$.
More details on our multi-initialization strategy are provided in Appendix~\ref{apdx:init_strategy}.

\subsection{Analysis of the Gap}

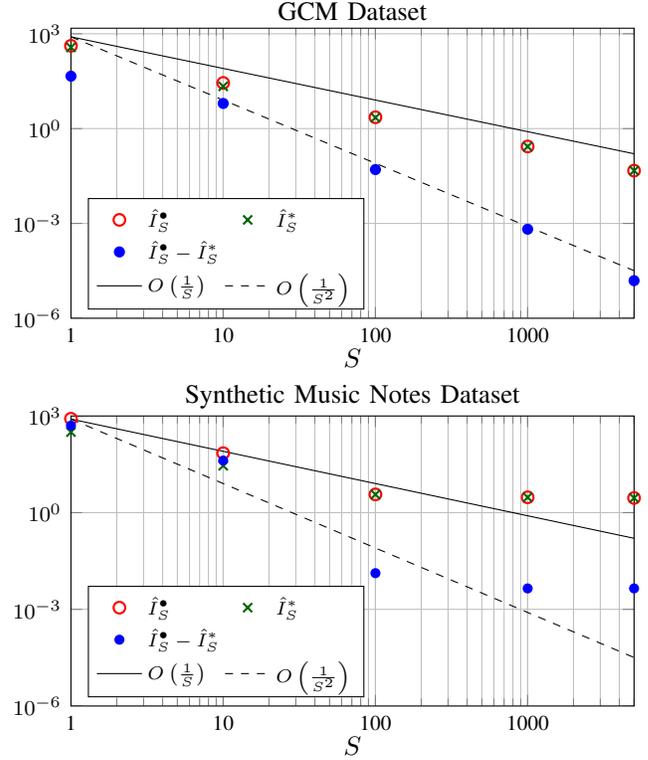
\begin{figure}[t]
    \centering
	\begin{tikzpicture}
	\begin{groupplot}[group style={group size= 1 by 2,       % taille du group plot
     				horizontal sep=0.5cm, vertical sep=1.3cm}, 
     				legend pos= south west,        % Position de la légende dans les figures
     				legend columns=2, 
 					legend style={legend cell align=left,font=\scriptsize},
 					grid=both,                         % grid
 					xmode=log,ymode=log,
 					xmin=1,xmax=5000,
 					xticklabels ={1,10,100,1000},
 				    title style={yshift=-1.5ex,},
 					tick label style={font=\footnotesize},
 			        xlabel style={yshift=1ex,},
     				height=0.3\textwidth,width=0.5\textwidth] 
		\nextgroupplot[ymin=0.000001,ymax=1500,title={GCM Dataset},xlabel={$S$}] 
		    \addplot[red,thick,mark=o,only marks,mark size=2.2pt] table{figs/ISdot_GCM.dat};\addlegendentry{$\dotr{\hat{I}}_S$};
			\addplot[darkgreen,thick,mark=x,only marks,mark size=2.4pt] table{figs/ISstar_GCM.dat};\addlegendentry{$\hat{I}_S^\ast$};
		    \addplot[blue,thick,mark=*,only marks,mark size=1.7pt] table{figs/Gap_IS_GCM.dat};\addlegendentry{$\dotr{\hat{I}}_S - \hat{I}_S^\ast$};\addlegendimage{empty legend};\addlegendentry{};
		    %\addplot[red,thick,mark=diamond,only marks,mark size=2.5pt] table{figs/QS_GCM.dat};
		    \addplot[domain=1:5000,black]    {800/x};\addlegendentry{$O\left(\frac{1}{S}\right)$};
		    \addplot[domain=1:5000,black,dashed]    {800/x^2};\addlegendentry{$O\left(\frac{1}{S^2}\right)$};
		\nextgroupplot[ymin=0.000001,ymax=1000,title={Synthetic Music Notes Dataset},xlabel={$S$}] 
		    \addplot[red,thick,mark=o,only marks,mark size=2.2pt] table{figs/ISdot_Notes.dat};\addlegendentry{$\dotr{\hat{I}}_S$};
  			\addplot[darkgreen,thick,mark=x,only marks,mark size=2.4pt] table{figs/ISstar_Notes.dat};\addlegendentry{$\hat{I}_S^\ast$};
		    \addplot[blue,thick,mark=*,only marks,mark size=1.5pt] table{figs/Gap_IS_Notes.dat};
		    \addlegendentry{$\dotr{\hat{I}}_S - \hat{I}_S^\ast$};
		    \addlegendimage{empty legend};\addlegendentry{};
		    \addplot[domain=1:5000,black]    {800/x};\addlegendentry{$O\left(\frac{1}{S}\right)$};
		    \addplot[domain=1:5000,black,dashed]    {800/x^2};\addlegendentry{$O\left(\frac{1}{S^2}\right)$};
 		\end{groupplot}
	\end{tikzpicture}
	\caption{\label{fig:evol_Is_with_S}
	Evolution of the empirical quantities  $\dotr{\hat{I}}_S$, $\hat{I}_S^\ast$, and $\dotr{\hat{I}}_S - \hat{I}_S^\ast$ as functions of the number of realizations $S$.
	}
\end{figure}

Next we will study the evolution of $\dotr{\hat{I}}_S$, $\hat{I}_S^\ast$, and $\dotr{\hat{I}}_S - \hat{I}_S^\ast$ as functions of the number of data realizations $S$. These quantities correspond to the evaluation of the function $I_S$ at the solution points $(\haPhi^\ast,\hW^\ast,\hHt^\ast)$  and $(\dotr{\haPhi},\dotr{\hW},\dotr{\hHt})$  obtained through the numerical resolution of TL-NMF and JD+NMF, respectively.
%Because Problems~\eqref{eq:OptimCS} and \eqref{eq:OptimLS+IS} are non-convex, those solution points are obtained from a multi-initialization strategy described in Appendix~\ref{apdx:init_strategy}.
Given our multi-initialization strategy, we may assume that $(\haPhi^\ast,\hW^\ast,\hHt^\ast)$ and $(\dotr{\haPhi},\dotr{\hW},\dotr{\hHt})$ belong to (or at least are close to) the solution sets $\Omega^*$ and $\dotr{\Omega}$.
Note that from Theorem~\ref{lem:DistSolSets}, Proposition~\ref{prop:Def_QS}, and by construction of $\hat{I}_S^\ast$ and $\dotr{\hat{I}}_S$, we expect the following inequalities
\begin{equation}\label{eq:ineq_expe}
    \underline{\lambda}^* \leq \hat{I}_S^\ast \leq  \bar{\lambda}^* \leq   \dotr{\underline{\lambda}} \leq  \dotr{\hat{I}}_S  \leq  \dotr{\bar{\lambda}} \leq \max_{\aPhi \in \St(M)} \, Q_S(\aPhi).
\end{equation}

Figure~\ref{fig:evol_Is_with_S} displays the values of $\dotr{\hat{I}}_S$, $\hat{I}_S^\ast$, and $\dotr{\hat{I}}_S - \hat{I}_S^\ast$ for the GCM and music notes datasets. For the GCM dataset, we observe that $\hat{I}_S^\ast$ decays as $O(1/S)$ which implies that all the quantities in~\eqref{eq:ineq_expe} (except $\underline{\lambda}^*$) enjoy at best the same rate of convergence. In particular, this constitutes an additional evidence that the bound derived in Remark~\ref{rmq:other_rate} is tight (also complementing Remark~\ref{rmq:tightness_rate}). 
It further shows that the bound in Proposition~\ref{prop:Def_QS} is also tight (under GCM). Concerning the  gap $\dotr{\hat{I}}_S - \hat{I}_S^\ast$, we observe in Figure~\ref{fig:evol_Is_with_S} that it converges at the faster rate of $O(1/S^2)$. From~\eqref{eq:innerBound} and~\eqref{eq:first_bound_gap}, this implies that the convergence rate  of the inner gap $ \dotr{\underline{\lambda}} - \bar{\lambda}^*$ is \textit{at least} of the order of $O(1/S^2)$ while the outer gap  $\dotr{\bar{\lambda}} - \underline{\lambda}^*$ converges at a rate that lies {between} $O(1/S^2)$ and $O(1/S)$. We can draw two possible hypothesis from this.  The first one is that both gaps would enjoy the same rate of convergence of $O(1/S^2)$ which would mean that our upper bound in~\eqref{eq:first_bound_gap} is loose (as we have seen that bounds in Proposition~\ref{prop:Def_QS} and Theorem~\ref{thm:pgcmrate} are tights). A direct analysis of $\dotr{\bar{\lambda}} - \underline{\lambda}^*$ rather than $\dotr{\bar{\lambda}}$ would thus be of interest. The second possible interpretation is that the two gaps would indeed not converge at the same rate and thus that a finer analysis would require to consider both gaps separately.

 Concerning the music notes dataset, 
 we observe that all the reported quantities are decaying up to $S \leq 100$. 
 However, as $S$ further grows, the quantity $\dotr{\hat{I}}_S$ seems to tend toward a constant. 

 Although a theoretical explanation of this phenomenon remains an open problem, it is likely due to the fact that the data do not exactly follow the conditions~\eqref{eq:GenModel}. Nevertheless, the solutions returned by TL-NMF and JD+NMF on this dataset remain very close when $S$ is large, as illustrated next.

   \begin{figure*}
       \centering
       \begin{tikzpicture}
       \node at (0,1.8) {TL-NMF};
       \node[rotate=90] at (-4.5,-1.5) {$S=1$};
       \node[rotate=90] at (-4.5,-7.5) {$S=100$};
       \node at (0,0) {\includegraphics[width=8cm]{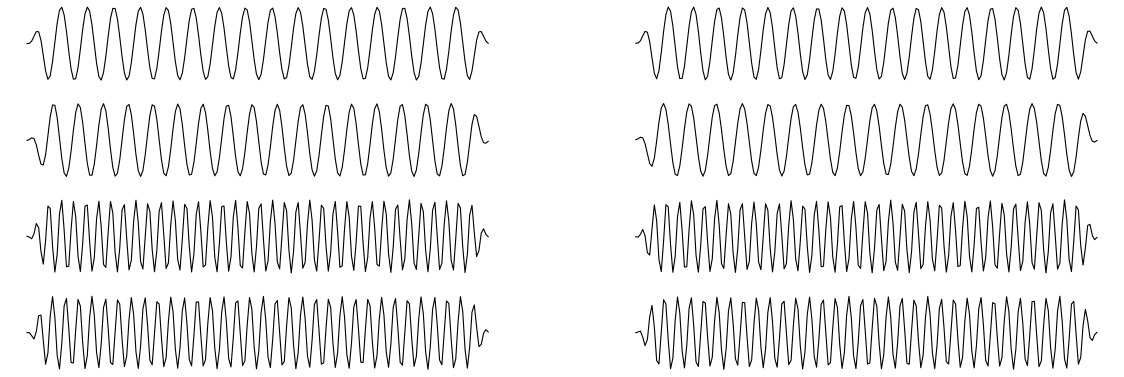}};
       \node at (0,-3) {\includegraphics[width=8cm]{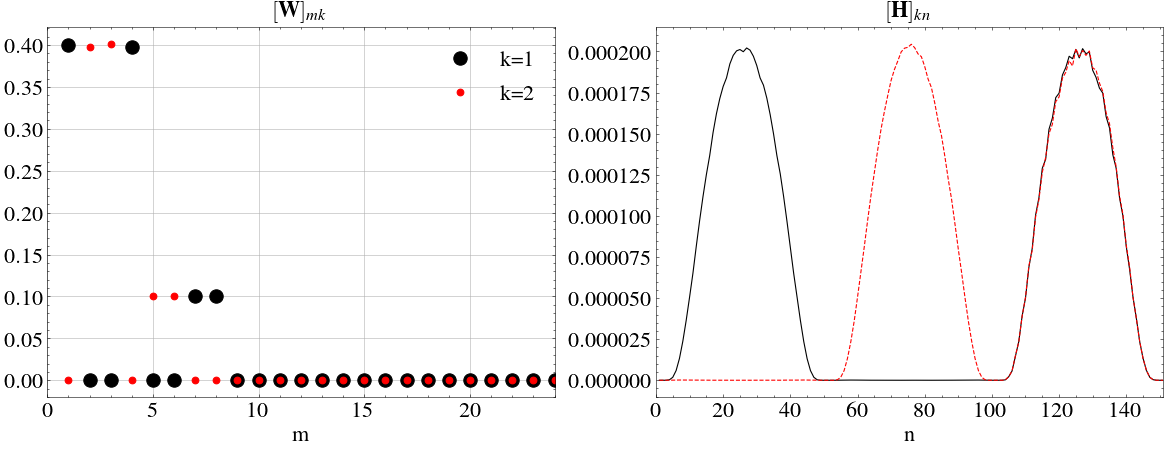}};       
       \node at (0,-6) {\includegraphics[width=8cm]{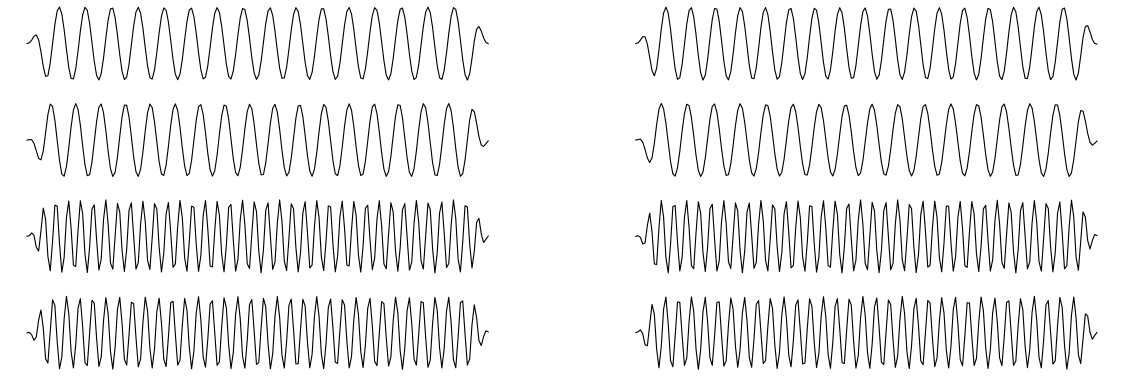}};
       \node at (0,-9) {\includegraphics[width=8cm]{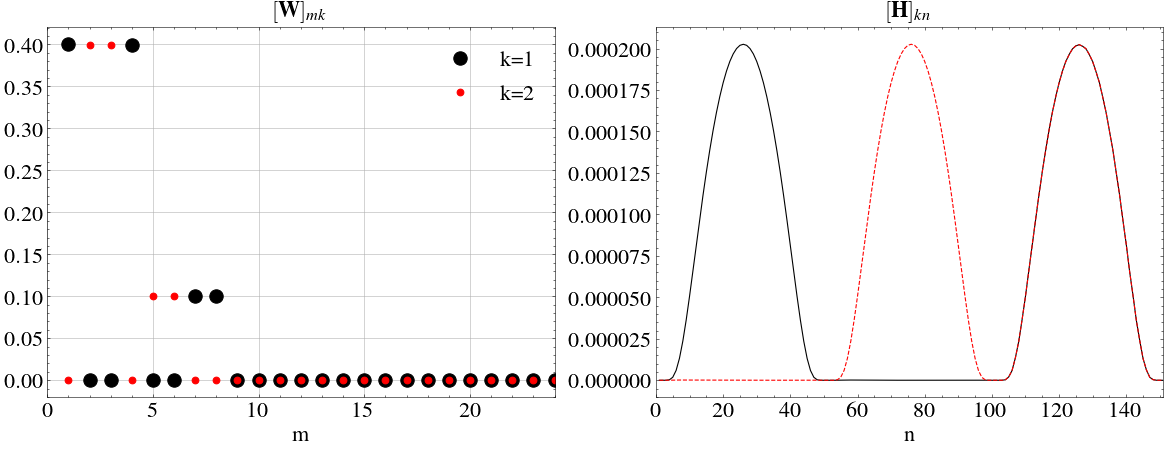}};       
       \node at (8.5,1.8) {JD+NMF};
       \node at (8.5,0) {\includegraphics[width=8cm]{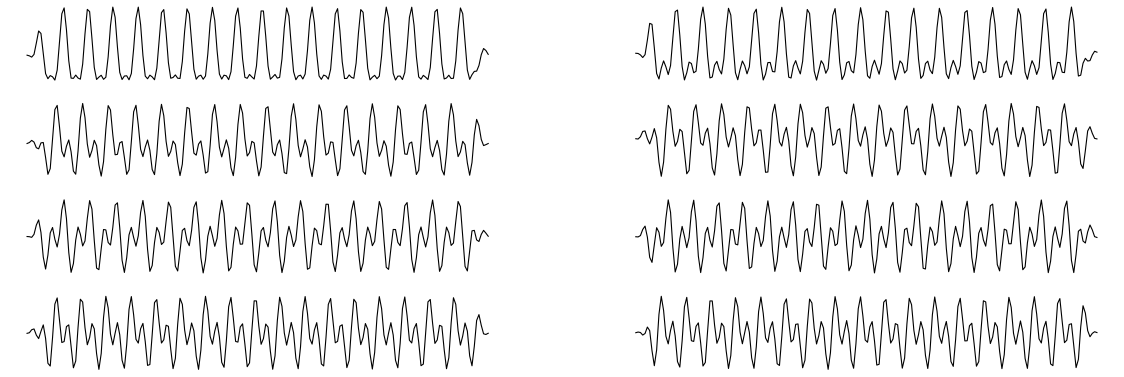}};
       \node at (8.5,-3) {\includegraphics[width=8cm]{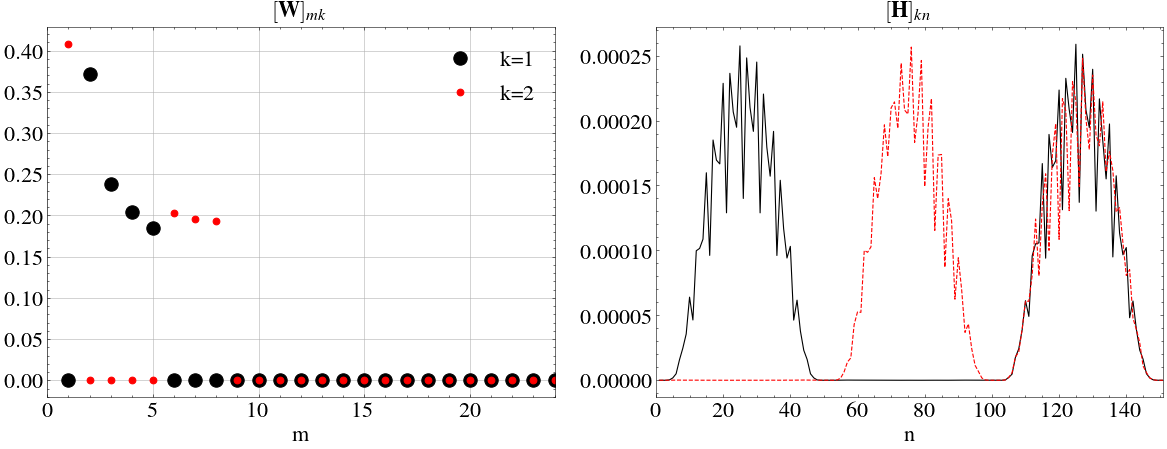}};       
       \node at (8.5,-6) {\includegraphics[width=8cm]{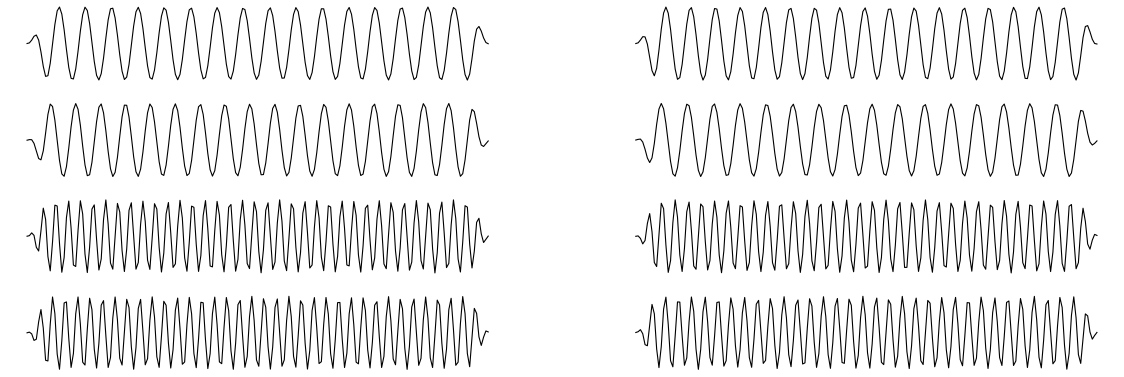}};
       \node at (8.5,-9) {\includegraphics[width=8cm]{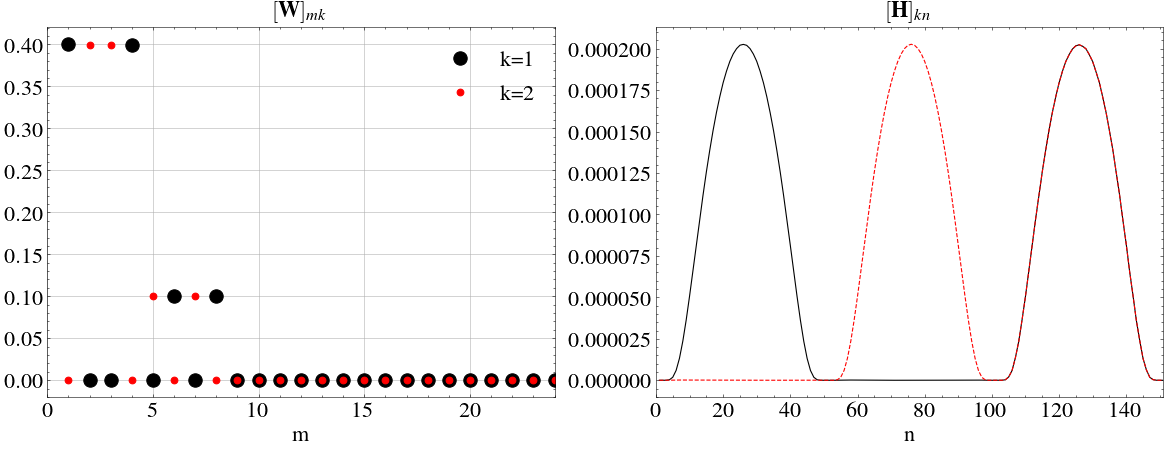}};       
	   \end{tikzpicture}
       \caption{\label{fig:sol_SPL}Plots of the eight most significant atoms (rows) of the learnt $\aPhi$ as well as the two columns of the estimated $\W$ and two rows of the estimated $\Ht$ for the music notes dataset. The most significant atoms are chosen as those that maximize $\E_S \| \underline{\bm{\phi}}_k \mathbf{Y} \|_2^2$, where $\underline{\bm{\phi}}_k$ is the $k^{th}$ row of $\aPhi$.
       }
   \end{figure*}
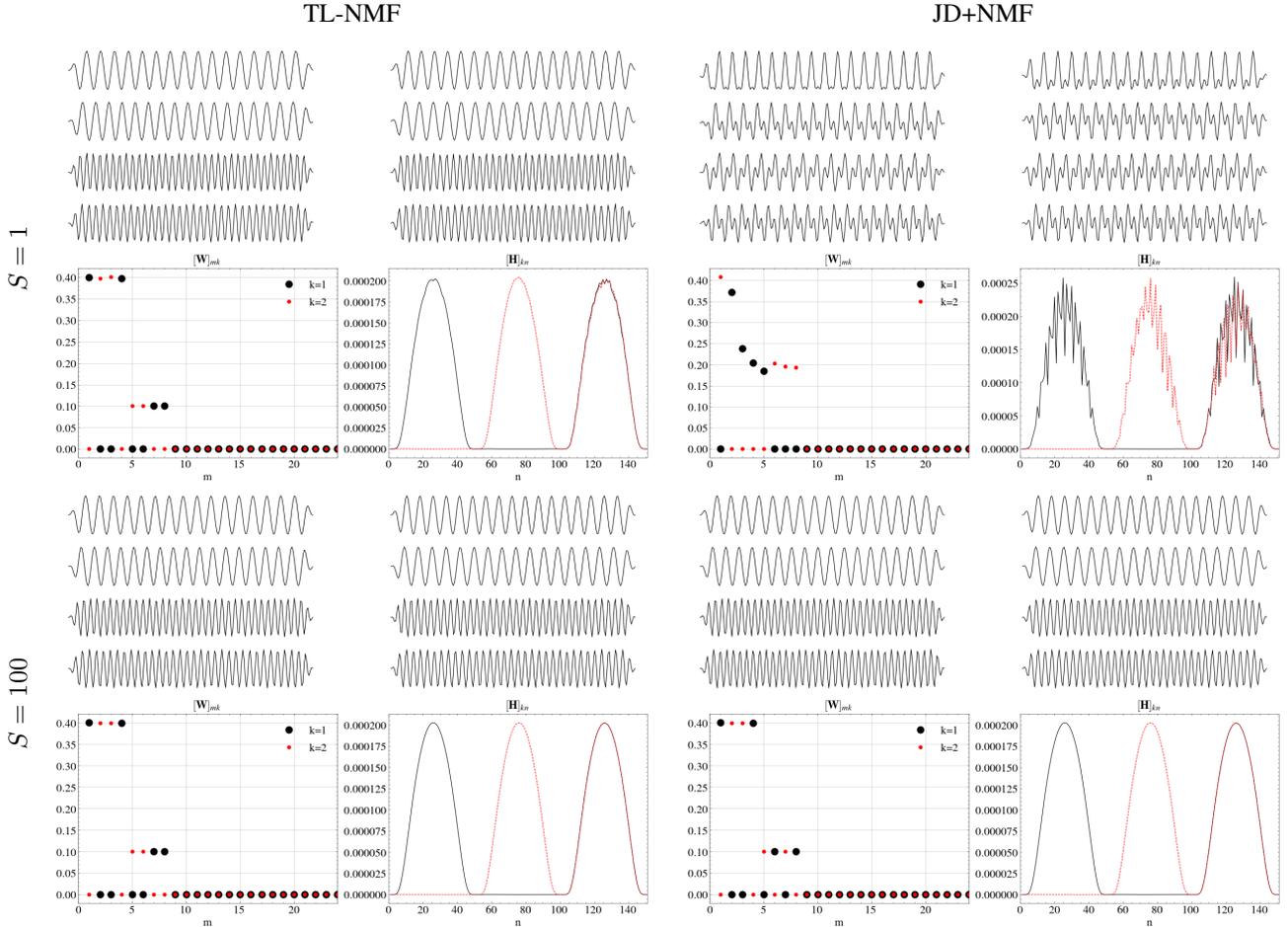
   
   \subsection{Analysis of the Solutions}  \label{sec:expe_S_small}
   
      In Figure~\ref{fig:sol_SPL}, we display the solutions obtained by TL-NMF and JD+NMF on the music notes dataset for both $S=1$ and $S=100$. As expected, when $S$ is large both approaches are able to learn a transform $\aPhi$ that captures the exact fundamental and harmonic frequencies of the two musical notes, allowing to significantly surpass the source separation performance of standard NMF that must abide to arbitrary frequency grid of the chosen DCT or Fourier frequency transform (see~\cite{Zhang2020On}). 
  
      However, the picture is quite different for $S=1$. Although the solution provided by TL-NMF appears to be as good as for $S=100$ (i.e., it enjoys the same favorable properties), the solution of JD+NMF is visibly degraded. 
      To further quantify the differences between the atoms learnt by TL-NMF and JD+NMF, we 
      performed a nonlinear least-square regression of the learnt atoms with the harmonic model  $(a,f,\theta)  \mapsto a \cos( 2 \pi \frac{f}{f_0} \cdot + \theta )$, as in \cite{Zhang2020On}. 
      We report in Table \ref{tbl:freq} the estimated frequency $f$ for an atom $\underline{\bm{\phi} } \in \mathbb{R}^M $, together 
      with the regression error $ \|\underline{\bm{\phi}} -  a \cos( 2 \pi \frac{f}{f_0} \cdot + \theta ) \|^2$. 
      We observe that the eight most significant atoms of TL-NMF ($S=1$) lead to a better fit to the fundamental frequencies of the two music notes due to smaller regression errors. 
      These errors could also explain why the NMF factors $\dotr{\W}$ and $\dotr{\Ht}$ are less regular in JD+NMF.
      These observations are in line with Figure~\ref{fig:evol_Is_with_S} where we see that, when $S$ is small, the numerical gap  $\dotr{\hat{I}}_S - \hat{I}_S^\ast$ is quite large, meaning that the transform $\aPhi^*$ learnt by TL-NMF performs better in 
       making $\E_S( [ \aPhi \Y ]_{mn}^{2} )$ more amenable to a low-rank NMF approximation
       than the transform $\dotr{\aPhi}$ obtained by JD+NMF.

\begin{table}
\small{
    \caption{Frequencies and regression errors of the TL-NMF and JD+NMF atoms when $S=1$ (same order as in Fig.~\ref{fig:sol_SPL}).
    }
    \label{tbl:freq}
	\begin{center}
    \begin{tabular}{ |c | c | }
    \hline
    \multicolumn{2}{|c|}{TL-NMF }  \\
    \hline
    Freq. & Error  \\
    \hline
    440.10 & 0.03  \\
    \hline
    466.35	 & 0.03  \\
    \hline
    466.11	 & 0.04   \\
    \hline
    439.74	 & 0.04  \\
    \hline
    932.32 & 0.03   \\
    \hline
    932.39	& 0.04  \\
    \hline
    879.94	 & 0.04  \\
    \hline
    879.99 & 0.03   \\
    \hline
    \end{tabular}
    \begin{tabular}{ |c||c| }
    \hline    
    \multicolumn{2}{|c|}{JD+NMF }  \\
    \hline
    Freq. & Error \\
    \hline
    466.12	 & 0.19  \\
    \hline
    439.72	 &  0.37 \\
    \hline
    879.97	 & 0.51 \\
    \hline
    879.99	 & 0.45 \\
    \hline
    879.97	 & 0.44 \\
    \hline
    932.46	 & 0.42 \\
    \hline
    932.28 &  0.40 \\
    \hline
    932.31 &  0.39 \\
    \hline
    \end{tabular}    
    \end{center}    
}
\end{table}

\subsection{Computational Complexity}\label{sec:complexity}

When JD+NMF provides solutions that are very close to those of TL-NMF (i.e., when $S$ is large), it is of interest to compare the computational load of the two methods.
To that end, we fix $J_\mathrm{TL}=1$, $J_\mathrm{NMF}=10$ and compute, for several values $J \in \mathbb{N}$,
\begin{itemize}
    \item $(\aPhi^*_J,\W^*_J,\Ht^*_J)$: the solution obtained after $J$ iterations of TL-NMF (each composed of $J_\mathrm{TL}$ updates of $\aPhi$ and $J_\mathrm{NMF}$ updates of ($\W,\Ht$)),
    \item $(\dotr{\aPhi}_J,\dotr{\W}_J,\dotr{\Ht}_J)$: the solution obtained after $(J \times J_\mathrm{TL})$ iterations of JD and $(J \times J_\mathrm{NMF})$ NMF updates.
\end{itemize}
By doing so, given $J \in \mathbb{N}$, the number of cumulative updates of $\aPhi$ and $(\W,\Ht)$ for both methods is the same and the comparison is fair.
%\es{Moreover, according to Remark~\ref{rmq:QN_solvers} and Section~\ref{sec:comp_cmplx}, we consider that both methods enjoy the same computational cost per iteration.}

We report in Figure~\ref{fig:complexity} the evolution of $C_S(\aPhi^*_J,\W^*_J,\Ht^*_J)$ and $C_S(\dotr{\aPhi}_J,\dotr{\W}_J,\dotr{\Ht}_J)$ as a function of $J$ for $S=100$. While on the music notes dataset both methods converge at a similar speed, we can see that JD+NMF converges faster than TL-NMF on the GCM dataset. This suggests that, in the situation where JD+NMF meets TL-NMF, the former should be preferred from the perspective of the computational complexity.

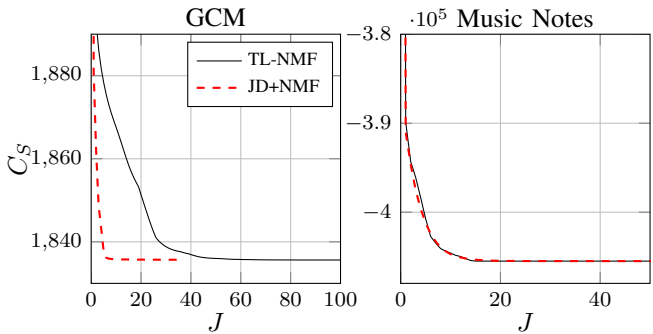
\begin{figure}[t]
    \centering
    \begin{tikzpicture}
    \begin{groupplot}[group style={group size= 2 by 1,       % taille du group plot
     				horizontal sep=0.8cm, vertical sep=0.3cm}, 
     				legend pos= north east,        % Position de la légende dans les figures
 					legend style={legend cell align=left,font=\scriptsize},
 					grid=both,                         % grid
 					%xmode=log,
 					%ymode=log,
 				%	xmin=1,xmax=5000,
 				%	yticklabels ={,,},
 				    title style={yshift=-1.5ex,},
 					tick label style={font=\footnotesize},
 			        xlabel style={yshift=1ex,},
 			        ylabel style={yshift=-1.5ex,},
     				height=0.27\textwidth,width=0.27\textwidth] 
		\nextgroupplot[xmin=0,xmax=100,ymin=1830,ymax=1890,title={GCM},xlabel={$J$},ylabel={$C_S$}] 
	    \addplot[black] table{figs/GCM_complexity_S100_tlnmf.dat};
		\addplot[red,dashed,thick] table{figs/GCM_complexity_S100_jdnmf.dat};
		\legend{TL-NMF,JD+NMF}
		\nextgroupplot[xmin=0,xmax=50,ymin=-408000,ymax=-380000,title={Music Notes},xlabel={$J$}] 
		\addplot[black] table{figs/Notes_complexity_S100_tlnmf.dat};
		\addplot[red,dashed,thick] table{figs/Notes_complexity_S100_jdnmf.dat};
 		\end{groupplot}
    \end{tikzpicture}
    \caption{\label{fig:complexity} Evolution of $C_S$ (for $S=100$) as a function of the parameter $J$ in Algorithms~\ref{alg:tl-nmf} and~\ref{alg:jd+nmf}.
    }
    
\end{figure}

\section{Discussion and Concluding Remarks} \label{sec:conc}

From our theoretical and numerical analysis, we can describe the relationship between TL-NMF and JD+NMF by distinguishing two situations.
On the one hand, when $S$ is large and $K \geq \bar{K}$, the two problems are equivalent (Theorems~\ref{lem:DistSolSets} and~\ref{thm:AssympCV}) and---at least on the reported numerical experiments---JD+NMF seems to enjoy a faster convergence (Figure~\ref{fig:complexity}). On the other hand, when $S$ is small (in particular $S=1$), they are not equivalent anymore and the solutions obtained by TL-NMF appear to be more meaningful from an application point of view (Figure~\ref{fig:sol_SPL}). One explanation of this phenomenon comes from the inherent low-rank constraint in the $\aPhi$ update of TL-NMF which favors the learning of orthogonal transforms that are better suited to a low-rank NMF. This is clearly visible when comparing the values of $I_S$ at the obtained solutions (see Figure~\ref{fig:evol_Is_with_S} and Theorem~\ref{lem:DistSolSets}).

% \sixin{
% \sixin{
% when $K$ is not too small (larger than $\bar{K}$}
% In this work, we focus on the regime where $K$ 
% is not too small (larger than $\bar{K}$). It remains an open problem to investigate
% the regime where $K < \bar{K}$. 

% In practice, $\bar{K}$ do not always exist since the condition \eqref{eq:GenModel} 
% is not always 
% }

In this work, we fixed $\E_S(|\aPhi \Y|^{\circ 2})$ to the empirical expectation given by~\eqref{eq:EmpiExpect}, which naturally led us to analyze the relationship between the two problems as a  function of the number of realizations $S$. However, we can see from Proposition~\ref{prop:Def_QS} that the main ingredient to make TL-NMF and JD+NMF equivalent is to build a good estimator of the variance $\E(|\aPhi \Y|^{\circ 2})$, in the sense that the Itakura-Saito divergence between the estimator and the true expectation is small (i.e., $Q_S(\aPhi)$ uniformly small). The empirical expectation is a natural choice when several signal realizations are available, but other choices are possible. Indeed, even for $S=1$, the variance $\E(|\aPhi \Y|^{\circ 2})$ can be estimated by a local moving average, under suitable local stationary assumptions. An example is when $\Ht$ is given a block structure (with identical columns inside the blocks), similarly to the \textit{block Gaussian model} of~\cite{PhamCardoso2001}. Actually, such a block structure can be seen alternatively as having access to several realizations of $\mathbf{Y}$ (with $S>1$ the size of the block).

In this paper we only considered real-valued transform for mathematical convenience. Complex-valued transform such as the short-time Fourier transform (STFT) are more customary in audio signal processing (in particular for shift-invariance). Extending our work to such complex-valued transform is an appealing research direction. Note also that the complex-valued Fourier transform is an orthogonal transform endowed with a specific form of Hermitian symmetry. In particular, when applied to real-valued signals, only half of the power spectrum needs to be considered. Imposing such a structure to $\boldsymbol{\Phi}$, to mimic some of the properties of the Fourier transform, would be a very interesting topic. This is however a challenging problem, both in terms of mathematical analysis and design of optimisation methods. Another exciting research direction would be to lift the orthogonal constraint imposed on $\boldsymbol{\Phi}$ and only assume invertibility.
However, for TL-NMF to be useful to signal processing tasks,
it is important that $\aPhi$ can be inverted for temporal reconstruction and thus that its inverse be well-conditioned. Existing JD algorithms such that \cite{Souloumiac2009,Bouchard2020} might be useful for this purpose.
% However, for TL-NMF to be applicable to source separation,
% it is important that $\aPhi$ is well-conditioned 
% as it involves applying $\aPhi^{-1}$ to NMF matrices which often have approximation or numerical errors. 
% Existing JD algorithms \cite{Souloumiac2009,Bouchard2020} might be useful for this purpose.
% ,  which needs to be 
% for it to be 
% Existing JD methods without the orthogonality constraint (such as Algos based on linear group ... CITE)
% do not explicitly control the condition number of the transform $\aPhi$, and therefore
% it may result in an unstable inverse $\aPhi^\intercal$, being sensitive to noise or numerical errors. 

% \change{with stable inverse}.  %  fully lift the orthogonal constraint imposed on $\boldsymbol{\Phi}$.}

% Considering the empirical expectation from a large number of samples $S$ is thus a natural choice to achieve this goal, but it is not the only one. 

% Indeed, even for $S=1$,  better estimators than the empirical expectation can be defined under some prior knowledge on the variance profiles (e.g., with a block structure along the columns of $\Ht$). A prime example being the \textit{block Gaussian model}~\cite{PhamCardoso2001}, where the variance profiles are assumed to be constant by blocks.

% \input{app}

% \newpage
% \section{Appendix}

\appendices

%% Proof Link LS and CS
\section{Proof of Lemma \ref{lemma:LinkCsLs}} \label{apdx:Proof_LinkCsLs}
% \es{[J'ai remplacé les $\X$ directement par de $\aPhi \Y$ finalement. Ok ?]}

%Let $\X = \aPhi \Y$ and $\x_n$ be the $n$-th column of $\X $. Then, 
The objective $L_S$ in~\eqref{eq:jd} can be rewritten as 
\begin{align}
 %	L_{S} ( \aPhi)  & = \sum_{m,n=1}^{M,N}
 %	\left(1  + \log  \E_S \left([\X]_{m n} ^2 \right) \right)  \nonumber \\
 %	& = MN + \sum_{n=1}^N \log \left(\prod_{m=1}^M  \E_S \left([\X]_{mn} ^2 \right)  \right)  \label{eq:combin1}
  	L_{S} ( \aPhi)  
%   	& = \sum_{m,n=1}^{M,N}
%  	\left(1  + \log   \left( \E_S \left([\aPhi \Y]_{m n} ^2 \right)   + \eps \right) \right)  \nonumber \\
 	& = MN + \sum_{n=1}^N \log \left(\prod_{m=1}^M  \left( \E_S \left([\aPhi \Y]_{mn} ^2 \right) + \eps \right) \right)  \label{eq:combin1}
%   	\\
%  	& = MN + \sum_{n=1}^N \log \left(\prod_{m=1}^M  \left[\E_S \left(\x_n\x_n^\intercal\right)\right]_{mm}  \right) \\
%   	& = MN + \sum_{n=1}^N  \log  \det \D \left(\E_S \left(\x_n\x_n^\intercal\right)\right).
\end{align}

Then from the definition of the empirical expectation %estimator
in \eqref{eq:EmpiExpect} and the orthogonality of $\aPhi$, 
we have that for every $(m,n)$, 
\begin{align}
%     \E_S \left([\X]_{mn} ^2 \right)  
 %    &= \left[ \frac{1}{S} \sum_{s=1}^S \x_n^{(s)} (\x_n^{(s)})^\intercal + \epsilon_S \Id \right]_{mm} \nonumber  \\
  %   & \underset{\eqref{eq:EmpiCov}}{=} \left[ \aPhi \bm{\Sigma}_{n,S} \aPhi^\intercal \right]_{mm} .  \label{eq:combin2}
         \E_S \left([\aPhi \Y]_{mn} ^2 \right)  + \eps
     &= \left[ \frac{1}{S} \sum_{s=1}^S (\aPhi \mathbf{y}_n^{(s)})(\aPhi \mathbf{y}_n^{(s)})^\intercal + \eps \Id \right]_{mm} \nonumber  \\
     & \underset{\eqref{eq:EmpiCov}}{=} \left[ \aPhi ( \bm{\Sigma}_{n,S}  + \eps \Id ) \aPhi^\intercal \right]_{mm} .  \label{eq:combin2}
\end{align}
Therefore combining the formula \eqref{eq:combin2} with \eqref{eq:combin1}, we have
\begin{equation*}
 	L_{S} ( \aPhi)   = MN + \sum_{n=1}^N  \log  \det \D \left( \aPhi (  \bm{\Sigma}_{n,S} + \eps \Id) \aPhi^\intercal \right).
 \end{equation*}
% Then using the orthogonality of  $\aPhi$, we have
% \begin{align*}
%     \E_S \left(\x_n\x_n^\intercal\right) &= \frac{1}{S} \sum_{s=1}^S \x_n^{(s)} (\x_n^{(s)})^\intercal + \epsilon_S \Id \\
%     & =  \frac{1}{S} \sum_{s=1}^S \aPhi \yd_n^{(s)} (\yd_n^{(s)})^\intercal \aPhi^\intercal + \epsilon_S \aPhi \aPhi^\intercal \\
%     & = \aPhi  \E_S \left(\yd_n\yd_n^\intercal\right) \aPhi^\intercal = \aPhi  \Sigma_{n,S} \aPhi^\intercal
% \end{align*}
% which completes the proof of equality~\eqref{eq:LinkCsLs}. 
Finally, the JD criterion derived in \cite{PhamCardoso2001}, that is
\begin{multline}\label{eq:pham}
    \sum_{n=1}^{N}   \log \det \D \left( \aPhi \left( \bm{\Sigma}_{n,S} + \eps \Id \right) \aPhi^\intercal \right)  - \\
                           \log \det \left( \aPhi   \left((\bm{\Sigma}_{n,S}  + \eps \Id ) \aPhi^\intercal \right)  \right),
\end{multline}
equals $L_S$ up to a constant term because, for $\aPhi$ orthogonal, $\det ( \aPhi   ( \bm{\Sigma}_{n,S} + \eps \Id ) \aPhi^\intercal ) = \det(\bm{\Sigma}_{n,S}+ \eps \Id)$.

%% Proof of Existence solutions
\section{Proof of Theorem~\ref{thm:ExistenceMinCs}} \label{apdx:Proof_ExistenceMinCs}

\subsection{Existence of a Solution for JD+NMF (Problem~\eqref{eq:OptimLS+IS})}

Given the two-step nature of Problem~\eqref{eq:OptimLS+IS}, the existence of a solution $(\dotr{\aPhi},\dotr{\W},\dotr{\Ht}) \in \dotr{\Omega}$ is a direct consequence of the following two lemmas.
\begin{lemm}\label{lem:ExistenceMinLs}
    The set of global minimizers of $L_S$ over $\St(M)$ is nonempty and compact.
\end{lemm}
\begin{proof}
As $\epsilon_0 > 0$, $\log  (  \E_S ( [\aPhi \Y]_{mn}^{2} ) + \eps ) $ is well defined and
$L_S$ is continuous on the compact manifold  $\St(M)$. 
Therefore $\inf_{\aPhi \in \St(M)} \; L_S(\aPhi)$ is attained and finite
(Weierstrass Theorem). Furthermore, the solution set is closed (it is a level set  of  the  continuous  function $L_S$) and bounded because it is included in the bounded set $\St(M)$. This proves the compactness and completes the proof.
\end{proof}

\begin{lemm}\label{lem:ExistenceMinIs}
For all $\aPhi \in \St(M)$, the set of global minimizers of $I_S(\aPhi,\cdot,\cdot)$ over $F_{K}$ is nonempty and compact.
\end{lemm}
\begin{proof}
By definition, for all $\aPhi \in \St(M)$, $I_S(\aPhi,\cdot,\cdot)$  is continuous over $F_{K}$. As 
$F_{K}$ is closed but unbounded, a sufficient argument to complete the proof is to show that, for all $\aPhi \in \St(M)$,   $I_S(\aPhi,\cdot,\cdot)$  is coercive over $F_{K}$, i.e.,
\begin{equation}\label{eq:Proof_ExistenceMinIs-1}
\forall \aPhi \in \St(M), \; \lim_{\substack{  \|(\W,\Ht)\|_{1,1}   \rightarrow \infty \\ (\W,\Ht) \in F_{K}}} 
I_S(\aPhi,\W,\Ht) = + \infty .
\end{equation}
Since all norms being equivalent in finite dimension, we use the norm $\|(\W,\Ht)\|_{1,1} = \|\W\|_{1,1} +\|\Ht\|_{1,1} $, where $\|\cdot\|_{1,1}$ denotes the entrywise $\ell_1$-norm for matrices.
Because of the normalization constraint on $\W$ in $F_{K}$, we have
\begin{equation}\label{eq:Proof_ExistenceMinIs-2}
\forall  (\W,\Ht) \in F_{K}, \;     \|(\W,\Ht) \|_{1,1} = K + \| \Ht \|_{1,1}.
\end{equation}

Now, let $\aPhi \in \St(M)$,  $\nu >0$ and set 
\begin{equation}\label{eq:Proof_ExistenceMinIs-3}
    \rho = KMN\mathrm{e}^{\nu + \log(D) +1} +K > K,
\end{equation}
where $D = \max_{m,n}  \E_S ( [\aPhi \Y]_{m, n }^{2} ) + \eps > 0$. Then we have the following implications:
\begin{multline}
(\W,\Ht) \in F_{K} \, \text{ and } \,  \|(\W,\Ht) \|_{1,1} \geq \rho \\
\quad \    \Longrightarrow \,  \| \Ht \|_{1,1} \geq \rho - K > 0 \\
    \Longrightarrow \,  \exists (k_0,n_0) \text{ s.t. } h_{k_0n_0} \geq \frac{\rho - K}{KN}, \label{eq:Proof_ExistenceMinIs-4}
\end{multline}
where the last implication comes from the fact that $\Ht \in \R_+^{K \times N}$. Similarly, using the fact that $\mathbf{w}_{k_0} \in \R_+^M$ and $\|\mathbf{w}_{k_0}\|_1=1$, we get that $\exists m_0$ s.t. $w_{m_0k_0}   \geq 1/M$. Combining the latter with~\eqref{eq:Proof_ExistenceMinIs-4}, we obtain
\begin{equation}\label{eq:Proof_ExistenceMinIs-5}
    [\W \Ht]_{m_0 n_0} \geq \frac{\rho - K}{MKN} \underset{\eqref{eq:Proof_ExistenceMinIs-3}}{\geq} \mathrm{e}^{\nu + \log(D) +1} .
\end{equation}
Finally, denoting $f(x)=x -\log(x) -1$ we have
\begin{align*}
	I_S(\aPhi, \W, \Ht) & = \sum_{m,n=1}^{M,N} f\left( \frac{ \E_S ( [\aPhi \Y]_{mn}^{2} ) + \epsilon_0 } { [\W \Ht]_{mn}  + \epsilon_0 }\right) \\
	& \geq f\left( \frac{ \E_S ( [\aPhi \Y]_{m_0n_0}^{2} ) + \eps } { [\W \Ht]_{m_0n_0} + \epsilon_0}\right) \quad \text{(as $f\geq0$ on $\R_+$)}\\
	& \geq - \log \left( \frac{ \E_S ( [\aPhi \Y]_{m_0n_0}^{2} ) + \eps } { [\W \Ht]_{m_0n_0} + \epsilon_0 } \right)  - 1  \\
	& \geq  \log( [\W \Ht]_{m_0n_0} + \epsilon_0) -  \log(D) -1 \underset{\eqref{eq:Proof_ExistenceMinIs-5}}{\geq} \nu. 
\end{align*}
Hence, we have shown that for all $\aPhi \in \St(M)$ and all  $\nu >0$, there exists $\rho >0$ (i.e.,~\eqref{eq:Proof_ExistenceMinIs-3}) such that
\begin{multline}
    (\W,\Ht) \in F_{K}  \text{ and }   \|(\W,\Ht) \|_{1,1} \geq \rho \\ \Longrightarrow  I_S(\aPhi, \W, \Ht) \geq \nu,
\end{multline}
which completes the proof.
\end{proof}

\subsection{Existence of a Solution for the TL-NMF Problem~\eqref{eq:OptimCS}} 
From Lemma~\ref{lem:ExistenceMinIs}, we get that the function 
\begin{equation}\label{eq:Proof_ExistenceMinCs-1}
    O_S(\aPhi) =  \min_{(\W,\Ht) \in F_{K}} \, I_S(\aPhi,\W,\Ht),
\end{equation}
is well defined on $\St(M)$ in the sense that the $\min$ exists for any $\aPhi \in \St(M)$. As such, from the decomposition of $C_S$ in~\eqref{eq:decompCs} and given that $\St(M)$ is compact and $L_S$ continuous, it is sufficient to prove that $O_S$ is continuous over $\St(M)$ (done by Lemma~\ref{lem:Cont_OS}) and invoke the Weierstrass Theorem to complete the proof that $\Omega^*$ is non-empty.

\begin{lemm}\label{lem:Cont_OS}
The function $O_S$ defined in~\eqref{eq:Proof_ExistenceMinCs-1} is continuous over $\St(M)$. 
\end{lemm}
\begin{proof}
We need to show that, for any $\aPhi\in \St(M)$ and $\Dd \in \mathcal{T}_\aPhi$ (the tangent space of $\St(M)$ at $\aPhi$),
\begin{equation}\label{eq:Proof_ExistenceMinCs-2}
    \lim_{\delta \rightarrow 0} \, |O_S(\pi(\aPhi + \delta \Dd ))  - O_S(\aPhi  ) | = 0,
\end{equation}
where $\pi$ stands for the projection operator on $\St(M)$~\cite{manton2002optimization}. 

Let $\aPhi\in \St(M)$, $\Dd \in \mathcal{T}_\aPhi$ and  $(\W_{\delta}, \Ht_{\delta}) \in F_{K}$ be a global minimizer of $I_S( \pi(\aPhi + \delta \Dd ),\cdot,\cdot)$ (there exists at least one from Lemma~\ref{lem:ExistenceMinIs}). Then,  
\begin{align}
    & O_S( \pi(\aPhi + \delta \Dd ) )  - O_S(\aPhi  ) \leq g_\delta(\W_{0},\Ht_{0}) \label{eq:Proof_ExistenceMinCs-3} \\
    & O_S( \pi(\aPhi + \delta \Dd ) )  - O_S(\aPhi  ) \geq g_\delta(\W_{\delta},\Ht_{\delta}), \label{eq:Proof_ExistenceMinCs-4}
\end{align}
with $g_\delta(\W,\Ht) =  I_S( \pi( \aPhi + \delta \Dd ),\W,\Ht ) - I_S(\aPhi ,\W,\Ht )$.
Moreover, from the definition of $I_S$ in~\eqref{eq:isnmf}--\eqref{eq:IsDivergence}, we get that for all $(\W,\Ht) \in F_{K}$,
\begin{multline}\label{eq:Proof_ExistenceMinCs-5}
    | g_\delta(\W,\Ht) |\leq \sum_{m,n=1}^{M,N}  
  \left| \log \left( \frac{ \E_S [ \pi(\aPhi + \delta \Dd) \Y]_{mn}^{2} +\eps  } {  \E_S [\aPhi \Y]_{mn}^{2}   +\eps } \right) \right| \\ +  \frac{1}{  \eps } \left| \E_S ( [ \pi(\aPhi + \delta \Dd) \Y]_{mn}^{2} ) -\E_S ( [\aPhi \Y]_{mn}^{2} ) \right|  . 
\end{multline}
Because the right-hand side is independent of $(\W,\Ht)$ and converges to $0$ as $\delta$ tends to $0$, we obtain 
\begin{equation}\label{eq:Proof_ExistenceMinCs-6}
   % \lim_{\delta \rightarrow 0} \; \sup_{(\W,\Ht) \in F_{K}} \, |g_\delta(\W,\Ht) | =0.
   \lim_{\delta \rightarrow 0} |g_\delta(\W_0,\Ht_0) | =0,\quad 
   \lim_{\delta \rightarrow 0} |g_\delta(\W_\delta,\Ht_\delta) | =0 . 
\end{equation}
Finally, combining~\eqref{eq:Proof_ExistenceMinCs-6}  with~\eqref{eq:Proof_ExistenceMinCs-3}--\eqref{eq:Proof_ExistenceMinCs-4}  proves~\eqref{eq:Proof_ExistenceMinCs-2}. % and completes the proof.
\end{proof}

\subsection{Compactness of $\dotr{\Omega}$} 
Given that $\dotr{\Omega}$ is finite dimensional, it is sufficient to show,  according to Heine–Borel theorem, that it is a non-empty (already proved), closed, and bounded set.

\subsubsection{Closedness} Let $\{(\aPhi_j,\W_j,\Ht_j)  \in \dotr{\Omega}\}_j$ be a sequence of $\dotr{\Omega}$ that converges toward $(\widehat{\aPhi},\widehat{\W},\widehat{\Ht})$. Then, from the continuity of $I_S$ and the fact that (by definition of $\dotr{\Omega}$), 
\begin{equation}
  \mkern-10mu  I_S(\aPhi_j,\W_j,\Ht_j) \leq I_S(\aPhi_j,\W,\Ht), \; \forall (\W,\Ht) \in  F_{K},
\end{equation}
we obtain
\begin{equation} \label{eq:Proof_CompactnessOmegaDot-1}
    I_S(\widehat{\aPhi},\widehat{\W},\widehat{\Ht}) \leq I_S(\widehat{\aPhi},\W,\Ht), \; \forall (\W,\Ht) \in  F_{K}.
\end{equation}
Noticing that $\widehat{\aPhi} \in \argmin_{\aPhi \in \St(M)} L_S(\aPhi)$ (compactness of the set of minimizers of $L_S$, Lemma~\ref{lem:ExistenceMinLs}), we conclude from~\eqref{eq:Proof_CompactnessOmegaDot-1} that  $(\widehat{\aPhi},\widehat{\W},\widehat{\Ht}) \in \dotr{\Omega}$, which shows that $\dotr{\Omega}$ is closed. 

\subsubsection{Boundedness} Let us assume that $\dotr{\Omega}$ is unbounded. Hence,  there exists a sequence   $\{(\aPhi_j,\W_j,\Ht_j)  \in \dotr{\Omega}\}_j$ such that 
\begin{align} \label{eq:Proof_CompactnessOmegaDot-2}
   \lim_{j \rightarrow \infty} \| (\aPhi_j,\W_j,\Ht_j)  \|_{1,1} = + \infty.
\end{align}
Then, using the fact that $\St(M)$ is compact, we get that the norm of $\aPhi_j$ is bounded and thus,  with~\eqref{eq:Proof_ExistenceMinIs-2} and~\eqref{eq:Proof_CompactnessOmegaDot-2},  that $\lim_{j \rightarrow \infty} \| \Ht_j  \|_{1,1} = + \infty$. It follows from the coercivity of $I_S$ (see Lemma~\ref{lem:ExistenceMinIs}) that
\begin{equation}
    \lim_{j \rightarrow \infty} O_S (\aPhi_j) \underset{\eqref{eq:Proof_ExistenceMinCs-1}}{=} \lim_{j \rightarrow \infty} I_S (\aPhi_j,\W_j,\Ht_j) = + \infty,
\end{equation}
which contradicts the fact that $O_S$ is a continuous function over the compact set $\St(M)$ (see Lemma~\ref{lem:Cont_OS}).

\subsection{Compactness of $\Omega^*$} 
%  \ced{[même rqe qu'avant]}
Again, because $\Omega^*$ is finite dimensional,  it is sufficient to show that $\Omega^*$ is a non-empty (already proved), closed, and bounded set.  Here, we directly get that $\Omega^*$ is closed as a level set of the continuous function $C_S$. Then,  let us assume that  $\Omega^*$ is unbounded. Hence, using the same arguments as in the proof for $\dotr{\Omega}$, this means that there exists a sequence   $\{(\aPhi_j,\W_j,\Ht_j)  \in \Omega^*\}_j$ for which  $\lim_{j \rightarrow \infty} \| \Ht_j  \|_{1,1} = + \infty$. Hence, because
\begin{equation}
  C_S( \aPhi_j, \W_j,\Ht_j ) \geq   \sum_{m,n=1}^{M,N} \log  \left(  [\W_j \Ht_j]_{mn}  + \eps \right) , 
\end{equation}
we have %\ced{$\lim_{j \rightarrow \infty} C_S (\aPhi_j,\W_j,\Ht_j) = + \infty$}, 
\begin{equation}
     \lim_{j \rightarrow \infty} C_S (\aPhi_j,\W_j,\Ht_j) = + \infty,
\end{equation}
which contradicts the fact that $\forall j$, $(\aPhi_j,\W_j,\Ht_j)  \in \Omega^*$.

\section{Proof of Theorem~\ref{lem:DistSolSets}}\label{apdx:Proof_DistSolSets}

First of all, the continuity of $I_S$ together with the compactness of $\Omega^*$ and $\dotr{\Omega}$ ensures the existence of the $\max$ and $\min$ in~\eqref{eq:lamb_under}--\eqref{eq:lamb_bar} (Weierstrass Theorem). 
The following three proofs rely on the key Equation~\eqref{eq:decompCs} which states that $C_S = L_S + I_S$. 
 
\paragraph*{Proof of~\eqref{eq:IneqLevSet}} By definition, we have $\underline{\lambda}^* \leq \bar{\lambda}^*$ and $\dotr{\underline{\lambda}} \leq \dotr{\bar{\lambda}}$. Moreover, the lower bound $0$ in~\eqref{eq:IneqLevSet} is due to the fact that $I_S \geq 0 $. 
It thus remains to show that $\bar{\lambda}^* \leq \dotr{\underline{\lambda}}$. Let us assume that $\bar{\lambda}^* > \dotr{\underline{\lambda}}$. Hence there exist $(\aPhi^*,\W^*,\Ht^*) \in \Omega^*$ and $(\dotr{\aPhi},\dotr{\W},\dotr{\Ht}) \in \dotr{\Omega}$ such that 
\begin{align}
        & I_S(\aPhi^*,\W^*,\Ht^*)   > I_S(\dotr{\aPhi},\dotr{\W},\dotr{\Ht}) \\
        \underset{\eqref{eq:decompCs}}{\Longleftrightarrow}  \; & C_S(\aPhi^*,\W^*,\Ht^*)   > L_S(\aPhi^*) + I_S(\dotr{\aPhi},\dotr{\W},\dotr{\Ht}) \\
        \Longrightarrow \; & C_S(\aPhi^*,\W^*,\Ht^*)   > C_S(\dotr{\aPhi},\dotr{\W},\dotr{\Ht}),
\end{align}
where the last implication comes from the fact that $L_S(\aPhi^*) \geq L_S(\dotr{\aPhi})$. This contradicts the fact that $(\aPhi^*,\W^*,\Ht^*)$ is a global minimizer of $C_S$ and completes the proof of~\eqref{eq:IneqLevSet}. 

\paragraph*{Proof of~\eqref{eq:CNS_partialIncl} and~\eqref{eq:CNS_fullIncl}} The inclusion $\Longrightarrow$ is straightforward from~\eqref{eq:IneqLevSet} together with the definition of $\bar{\lambda}^* $ and $\dotr{\underline{\lambda}}$. Now, if $\bar{\lambda}^* = \dotr{\underline{\lambda}}$, then  there exist $(\aPhi^*,\W^*,\Ht^*) \in \Omega^*$ and $(\dotr{\aPhi},\dotr{\W},\dotr{\Ht}) \in \dotr{\Omega}$ such that 
\begin{align}
        & I_S(\aPhi^*,\W^*,\Ht^*)   = I_S(\dotr{\aPhi},\dotr{\W},\dotr{\Ht}) \label{eq:proof_thDistSolSets-1}\\
        \underset{\eqref{eq:decompCs}}{\Longleftrightarrow}  \; & C_S(\aPhi^*,\W^*,\Ht^*)   = L_S(\aPhi^*) + I_S(\dotr{\aPhi},\dotr{\W},\dotr{\Ht}) \\
        \Longrightarrow \; & C_S(\aPhi^*,\W^*,\Ht^*)   \geq  C_S(\dotr{\aPhi},\dotr{\W},\dotr{\Ht}).
\end{align}
Given that  $(\aPhi^*,\W^*,\Ht^*)$ is a global minimizer of $C_S$, we have $C_S(\aPhi^*,\W^*,\Ht^*)   =  C_S(\dotr{\aPhi},\dotr{\W},\dotr{\Ht})$. Moreover, the latter equality  with~\eqref{eq:proof_thDistSolSets-1}  lead to $L_S(\aPhi^*) =L_S(\dotr{\aPhi})$. This completes the proof.

\paragraph*{Proof of~\eqref{eq:CNS_fullIncl}} can be obtained following the same steps.

%\subsection{Proof of~\eqref{eq:CNS_fullIncl}} Can be obtained following the same steps as the proof of~\eqref{eq:CNS_partialIncl}. 

\section{Proof of Proposition~\ref{prop:Def_QS}}\label{apdx:proof_Def_QS}

First of all, one can easily see that
\begin{equation}
      \dotr{\bar{\lambda}}  \leq  \max_{\aPhi \in \St(M)} \, O_S(\aPhi),
\end{equation}
where $O_S$ is defined as in~\eqref{eq:Proof_ExistenceMinCs-1}. To complete the proof, we shall show that, $\forall \aPhi \in \St(M)$, $O_S(\aPhi) \leq Q_S(\aPhi)$ where $Q_S$ is defined in~\eqref{eq:def_QS}.

To that end, let us show\footnote{This result was already shown in \cite{Zhang2020On} without the constraint set $ F_{K}$. For completeness, we recall all the steps here.} that, for any $\aPhi \in \St(M)$, there exist  $( \W_\aPhi^\ast , \Ht_\aPhi^\ast) \in F_{ K}$ such that $ \E ( [ \aPhi \Y ]_{mn}^{2}) = \W_\aPhi^\ast \Ht_\aPhi^\ast$.  Let $\mathbf{D} = \aPhi \bPhi^\intercal$, $\bar \V = \bar \W \bar \Ht$ with $\bar{v}_{mn} = [\bar \V ]_{mn}$ and $d_{mm'} = [\mathbf{D}]_{mm'}$. Then, from \eqref{zerocond} and \eqref{jdcond} we have 
\begin{align}
     \E ( [ \aPhi \Y ]_{mn}^{2})& = 
     \E \left( \left( \sum_{m'=1}^M d_{m m'} [ \bPhi \Y]_{m' n}  \right)^2 \right) \label{eq:exactFacto0} \\
     &  = \sum_{m'=1}^M d_{m m'}^2  \bar{v}_{m'n}  \\
     & = \sum_{k=1}^{\bar K} \sum_{m'=1}^M d_{m m'}^2 \bar{w}_{m' k} \bar{h}_{kn} 
      \label{eq:exactFacto}
\end{align}
For \eqref{eq:exactFacto0}, we used the orthogonality of $\bar \aPhi$ to get $\aPhi  = \mathbf{D} \bPhi$.
Defining $( \W_\aPhi^\ast , \Ht_\aPhi^\ast) \in  \R^{M \times  K}_{+} \times \R^{K \times N}_{+} $ such that 
\begin{align}
    & [ \W_\aPhi^\ast]_{mk} = \left\lbrace
    \begin{array}{ll}
        \sum_{m'} \frac{ \bar{w}_{m' k} }{\| \bar{\mathbf{w}}_k\|_1} d_{m m'}^2 & \text{ if } k \leq \bar K,  \\
        \frac{1}{M} &  \text{ if } \bar K <  k \leq K,
    \end{array}\right. \label{eq:exactFacto0-1}
\end{align}
\begin{align}
    & [ \Ht_\aPhi^\ast]_{kn} = \left\lbrace
    \begin{array}{ll}
         \bar{h}_{kn}\| \bar{\mathbf{w}}_k\|_1 & \text{ if } k \leq \bar K , \\
       0  &  \text{ if } \bar K <  k \leq K,
    \end{array}\right.
\end{align}
we get from~\eqref{eq:exactFacto} that
\begin{equation}\label{eq:proofTS-0}
\W_\aPhi^\ast \Ht_\aPhi^\ast = \E ( | \aPhi \mathbf{Y}|^{\circ 2}).
\end{equation}
Moreover, as $\sum_{m} d_{m m'}^2 = 1$, we obtain that, for all $k \leq \bar K$,
\begin{equation}
    \sum_{m=1}^M  [ \W_\aPhi^\ast]_{mk} =  \sum_{m'=1}^M \frac{ \bar{w}_{m' k} }{\| \bar{\mathbf{w}}_k\|_1} \sum_{m=1}^M d_{m m'}^2 = 1.
\end{equation}
Given that, by definition in~\eqref{eq:exactFacto0-1}, 
we also have $\sum_{m=1}^M  [ \W_\aPhi^\ast]_{mk} =1$ for $\bar K <  k \leq K$. 
Therefore we have shown that $( \W_\aPhi^\ast , \Ht_\aPhi^\ast) \in F_K$.
% we have shown that $( \W_\aPhi^\ast , \Ht_\aPhi^\ast) \in F_K$.

% Given that all the terms in~\eqref{eq:exactFacto} are non-negative and the rank of the matrix $\bar \V$ is at most $ \bar K \leq  K$, we have derived an exact NMF  of $\E ( | \aPhi \mathbf{Y}|^{\circ 2})$ that we denote from now on by $( \W_\aPhi^\ast , \Ht_\aPhi^\ast) \in  \R^{M \times  K}_{+} \times \R^{K \times N}_{+} $. 
% As $\sum_{m'} d_{m m'}^2 = 1$, we can further construct $( \W_\aPhi^\ast , \Ht_\aPhi^\ast) \in F_{ K}$ (up to the normalization of the columns of $\W_\aPhi^\ast$ and the appropriate inverse scaling of the rows of $\Ht_\aPhi^\ast$), 
% such that
% \begin{equation}\label{eq:proofTS-0}
%     [\W_\aPhi^\ast \Ht_\aPhi^\ast ]_{mn} =   \E ( [ \aPhi \Y ]_{mn}^{2}) , %   \geq \epsilon_0, 
% \end{equation}

% Note that as $\bar{v}_{m'n} > 0$ under the model, we introduce 
% $\eps_0 = \min_{m',n} \bar{v}_{m'n} $. 

It follows from~\eqref{eq:proofTS-0} that, $ \forall \aPhi \in \St(M)$,
\begin{align}
    O_S(\aPhi) &\leq I_S(\aPhi, \W_\aPhi^\ast , \Ht_\aPhi^\ast) \\
    & = D_\eps\left( \E_S(|\aPhi \Y|^{\circ 2}) | \W_\aPhi^\ast \Ht_\aPhi^\ast\right)
  = Q_S(\aPhi).
\end{align}
%which completes the proof.

\section{Proof of Theorems~\ref{thm:AssympCV} and~\ref{thm:pgcmrate}} \label{sec:AssympCV}
\subsection{Preliminaries} 

In Lemma~\ref{lem:Equiv_UniformCV_Es} we derive an equivalent formulation of condition~\eqref{erruniformTh} that is more convenient to prove Theorems~\ref{thm:AssympCV} and~\ref{thm:pgcmrate}.

\begin{lemm}\label{lem:Equiv_UniformCV_Es}
Condition~\eqref{erruniformTh} is satisfied if and only if
  $  \TS  \overset{p}{\to} 0 \; \text{ as } \; S \to \infty$ 
where 
\begin{align}
    \TS & = \max_{m,n}  \max_{\aPhi \in \St(M)} \left|  \frac{ \E_S ( [\aPhi \Y ]_{mn}^2 ) + \eps  } { \E ( [\aPhi \Y]^{2}_{mn} ) + \eps } -1   \right| \label{eq:Equiv_UniformCV_Es-1} \\
    & \leq \max_n \|   \bm{\Sigma}_n^{-1/2} \bm{\Sigma}_{n,S} \bm{\Sigma}_n^{-1/2} - \mathbf{I} \| . \label{eq:Equiv_UniformCV_Es-2}
\end{align}
\end{lemm}
\begin{proof}
Given the fact that $[\bar \W \bar \Ht]_{mn}>0$ $\forall (m,n)$, we introduce two constants $\epsilon_1>0$ and $\epsilon_2>0$ such that for all $(m,n)$
\begin{equation} \label{eq:Proof_lem_CondCV_Es-0}
     \epsilon_1 \leq  [ \bar \W \bar \Ht ]_{mn} \leq \epsilon_2. 
     % \|\bar \W \bar \Ht \|_{\max}, % \leq \| \bar \W \bar \Ht \|_F.
\end{equation}

Then, combining~\eqref{eq:exactFacto} with  $\sum_{m'} d_{m m'}^2 = 1$, one can see that for all $\aPhi \in \St(M)$ and $(m,n)$,
\begin{equation}
     \epsilon_1 \leq \E([\aPhi \Y]_{mn}^2) \leq  \epsilon_2 . % \leq \| \bar \W \bar \Ht \|_F.
\end{equation}
% where $\|\cdot\|_{\max}$ is the elementwise max norm for matrices \ced{[utiliser norm $\infty$ ?]}.
It follows that for any $\epsilon>0$ 
\begin{align}
     &   \left| \E_S ( [\aPhi \Y ]_{mn}^2 )  - \E ( [\aPhi \Y]^{2}_{mn} ) \right|  < \epsilon,  \label{eq:Proof_lem_CondCV_Es-1} \\
  \Longrightarrow \;   &  \left| \frac{\E_S([ \aPhi \Y ]^{ 2}_{mn}) +  \eps }{\E([ \aPhi \Y ]^{2}_{mn}) +  \eps } -1 \right|  <  \frac{\epsilon}{\epsilon_0},   \label{eq:Proof_lem_CondCV_Es-2} 
\end{align}
and conversely
\begin{align}
     &  \left| \frac{\E_S([ \aPhi \Y ]^{ 2}_{mn}) +  \eps  }{\E([ \aPhi \Y ]^{2}_{mn}) +  \eps } -1 \right| <  \epsilon,  \label{eq:Proof_lem_CondCV_Es-3}  \\
      \Longrightarrow \; &   \left| \E_S ( [\aPhi \Y ]_{mn}^2 )  - \E ( [\aPhi \Y]^{2}_{mn} ) \right|  < \epsilon (\epsilon_2 + \eps).  \label{eq:Proof_lem_CondCV_Es-4} 
\end{align}

By definition of the convergence in probability, condition~\eqref{erruniformTh} is equivalent to:  $\forall \epsilon >0$ and $\delta \in (0,1)$, there exists $S^\star$ such that $ \forall S \geq S^\star $
\begin{multline}
   \mathrm{Pr}\left( \max_{m,n}
\max_{\aPhi \in \St(M)} \left| \E_S ( [\aPhi \Y ]_{mn}^2 )  - \E ( [\aPhi \Y]^{2}_{mn} )  \right| < \epsilon \right) \\ > 1- \delta,
\end{multline}
which, using~\eqref{eq:Proof_lem_CondCV_Es-1}--\eqref{eq:Proof_lem_CondCV_Es-2}, implies that
\begin{equation}\label{eq:Proof_lem_CondCV_Es-5} 
   \mathrm{Pr}\left( \TS < \frac{\epsilon}{\epsilon_0} \right)  > 1- \delta.
\end{equation}
This proves that~\eqref{erruniformTh} $\Longrightarrow \left(  \TS  \overset{p}{\to} 0 \; \text{ as } \; S \to \infty \right)$. The converse can be proven similarly with~\eqref{eq:Proof_lem_CondCV_Es-3}--\eqref{eq:Proof_lem_CondCV_Es-4}.

To complete the proof, it remains to show the inequality~\eqref{eq:Equiv_UniformCV_Es-2}.
First of all, let us notice that, from~\eqref{jdcond} and~\eqref{eq:Proof_lem_CondCV_Es-0},  $\bm{\Sigma}_n$ is positive definite and thus $\bm{\Sigma}_n^{-1/2}$ is well defined. 
%It relies on the working assumption which implies that the inverse of $\bm{\Sigma}_n $ exists for all $n$ (\sixin{$\epsilon_1 > 0$}). 
%Indeed, w
Then, we have for all $(m,n)$
\begin{align}
&\max_{\aPhi \in \St(M)}  \bigg    | \frac{ \E_S ( [\aPhi \Y ]_{mn}^2 ) + \eps  } { \E ( [\aPhi \Y]^{2}_{mn} ) + \eps } -1  \bigg |  \\
 \underset{\eqref{eq:combin2}}{\leq} & \max_{\bm{\phi} \in \mathbb{R}^M : \| \bm{\phi} \|_2=1} \left|   
 \frac{ \bm{\phi}^\intercal  ( \bm{\Sigma}_{n,S}   - \bm{\Sigma}_n ) \bm{\phi} } {  \bm{\phi}^\intercal  \bm{\Sigma}_n  \bm{\phi}  }     \right| \\
= & \max_{\tilde{\bm{\phi}} \in \mathbb{R}^M : \| \bm{\Sigma}_n^{-1/2}\tilde{\bm{\phi}} \|_2=1 } 
 \left|   \frac{ {\tilde{\bm{\phi}}}^\intercal (  \bm{\Sigma}_n^{-1/2} \bm{\Sigma}_{n,S} \bm{\Sigma}_n^{-1/2} - \mathbf{I} )  \tilde{\bm{\phi}} } { \|  \tilde{\bm{\phi}} \|_2^2 }   \right| \\
  = & \max_{\tilde{\bm{\phi}} \in \mathbb{R}^M : \|\tilde{\bm{\phi}} \|_2=1 } 
 \left|   {\tilde{\bm{\phi}}}^\intercal (  \bm{\Sigma}_n^{-1/2} \bm{\Sigma}_{n,S} \bm{\Sigma}_n^{-1/2} - \mathbf{I} )  \tilde{\bm{\phi}}    \right| \\
 = & \; \|   \bm{\Sigma}_n^{-1/2} \bm{\Sigma}_{n,S} \bm{\Sigma}_n^{-1/2} - \mathbf{I} \|,
\end{align}
which completes the proof. 
\end{proof}

In Lemma~\ref{lem:Bound_Qs_with_Ts}, we show that $\max_{\aPhi \in \St(M)} Q_S(\aPhi)$ (defined in~\eqref{eq:def_QS}) can be controlled in terms of $\TS$ and that it converges to $0$ as $\TS$ tends to $0$.

\begin{lemm}\label{lem:Bound_Qs_with_Ts}
Let $\TS$ be defined as in Lemma~\ref{lem:Equiv_UniformCV_Es}. Then,
\begin{equation}
    \TS < 1 \Longrightarrow \max_{\aPhi \in \St(M)} Q_S (\aPhi) \leq MN \frac{\TS^2} { 1 - \TS}.
\end{equation}
\end{lemm}
\begin{proof}
Let us first recall that 
\begin{equation}\label{eq:Proof_Bound_Qs_with_Ts-1}
    Q_S(\aPhi) = \sum_{m,n=1}^{M,N} f\left(  \frac{\E_S([ \aPhi \Y ]^{ 2}_{mn}) + \eps }{\E([ \aPhi \Y ]^{2}_{mn}   ) + \eps } \right),
\end{equation}
for $f(x) = x - \log(x) -1$.
Then, by definition of $\TS$ in~\eqref{eq:Equiv_UniformCV_Es-1} we have that, for all $\aPhi \in \St(M)$ and $(m,n)$,
\begin{equation}\label{eq:Proof_Bound_Qs_with_Ts-2}
     \frac{\E_S([ \aPhi \Y ]^{ 2}_{mn}) + \eps }{\E([ \aPhi \Y ]^{2}_{mn} ) + \eps } \in \left[1-\TS, 1+\TS \right].
\end{equation}
Hence, because $\forall x >0, f(x) \leq g(x) :=x + 1/x - 2$, it follows from~\eqref{eq:Proof_Bound_Qs_with_Ts-1} and~\eqref{eq:Proof_Bound_Qs_with_Ts-2} that, if $\TS <1$, $\forall \aPhi \in \St(M)$
\begin{equation}\label{eq:Proof_Bound_Qs_with_Ts-3}
     Q_S(\aPhi) \leq \sum_{m,n=1}^{M,N} g \left(  \frac{\E_S([ \aPhi \Y ]^{ 2}_{mn}) + \eps }{\E([ \aPhi \Y ]^{2}_{mn}   ) + \eps} \right)
\end{equation}
and thus
\begin{align}
    \max_{\aPhi \in \St(M)} Q_S(\aPhi) & \leq MN \max_{x \in [1-\TS,1+\TS]} g(x) \label{eq:Proof_Bound_Qs_with_Ts-4}\\
    & = MN \max\{ g(1-\TS), g(1+\TS)\} \label{eq:Proof_Bound_Qs_with_Ts-5} \\
    & = MN g(1-\TS) = MN \frac{\TS^2}{1-\TS}. \label{eq:Proof_Bound_Qs_with_Ts-6}
\end{align}
Equality~\eqref{eq:Proof_Bound_Qs_with_Ts-5} comes from the fact that $g$ is a convex function  and~\eqref{eq:Proof_Bound_Qs_with_Ts-6} is due to  $ g(1-\TS) \geq  g(1+\TS)$ for $\TS \in [0,1)$.
\end{proof}

% NOTE : En fait on peut faire la même chose sans passer par la majorisation avec $g$ et travailler direct avec $f$. Dans ce cas on aurait la borne 
% \begin{equation}
%     MN f(1-\TS) = - MN (\log(1-\TS) + \TS)
% \end{equation}
% qui est plus tight mais qui me semble plus compliquer à utiliser à la fin avec le $h_S$.

\subsection{Proof of Theorem~\ref{thm:AssympCV}} \label{apdx:Proof_AssympCV}

The proof of Theorem~\ref{thm:AssympCV} is a direct consequence of Lemmas~\ref{lem:Equiv_UniformCV_Es} and~\ref{lem:Bound_Qs_with_Ts}. Indeed, under condition~\eqref{erruniformTh}, we get from Lemma~\ref{lem:Equiv_UniformCV_Es} that, for all $\epsilon \in (0, \epsilon_0)$ and $\delta \in (0,1)$, there exists $S^\ast$ such that
\begin{equation}
   S \geq S^\ast \quad   \Longrightarrow \quad  \mathrm{Pr}\left( \TS < \frac{\epsilon}{\epsilon_0} \right)  > 1- \delta.
\end{equation}
(cf.~\eqref{eq:Proof_lem_CondCV_Es-5}) which, with Lemma~\ref{lem:Bound_Qs_with_Ts}, completes the proof.

\subsection{Proof of Theorem~\ref{thm:pgcmrate}} \label{apdx:Proof_pgcmrate}

From Lemma~\ref{lem:Equiv_UniformCV_Es}, showing that condition~\eqref{erruniformTh} is satisfied under GCM  amounts to prove that under GCM we have $\TS  \overset{p}{\to} 0 \; \text{ as } \; S \to \infty$. Moreover, by making explicit a bound on the convergence rate of $\TS$, we will get a bound on the convergence rate of $\max_{\aPhi \in \St(M)} Q_S(\aPhi)$ thanks to Lemma~\ref{lem:Bound_Qs_with_Ts}. \\

According to Lemma~\ref{lem:Equiv_UniformCV_Es} (Equation~\eqref{eq:Equiv_UniformCV_Es-2}), to show the convergence of $\TS$ in probability, it is sufficient to show that each eigenvalue of $\bm{\Sigma}_n^{-1/2} \bm{\Sigma}_{n,S} \bm{\Sigma}_n^{-1/2}$ converges to one for each $n$. 

To that end, let us first remark that the normalized matrix $\bm{\Sigma}_n^{-1/2} \bm{\Sigma}_{n,S} \bm{\Sigma}_n^{-1/2}$ can be regarded as an empirical covariance matrix obtained from the $S$ ``whitened'' vectors  $ \{  \mathbf{x}_n^{(s)}  = \bm{\Sigma}_n^{-1/2} \mathbf{y}_n^{(s)} \}_{s=1}^S$, i.e.,
% \begin{align}
%       \bm{\Sigma}_n^{-1/2} \bm{\Sigma}_{n,S} \bm{\Sigma}_n^{-1/2} & = \frac{1}{S} \sum_{s=1}^S   \mathbf{x}_n^{(s)}  \mathbf{x}_n^{(s)\intercal}  \label{eq:Proof_pgcmrate-1} \\
%       & = \mathbf{A}_S^\intercal \mathbf{A}_S,\label{eq:Proof_pgcmrate-2} 
% \end{align}
\begin{align}
      \bm{\Sigma}_n^{-1/2} \bm{\Sigma}_{n,S} \bm{\Sigma}_n^{-1/2}  = \frac{1}{S} \sum_{s=1}^S   \mathbf{x}_n^{(s)}  \mathbf{x}_n^{(s)\intercal} = \mathbf{A}_S^\intercal \mathbf{A}_S,\label{eq:Proof_pgcmrate-2} 
\end{align}
where $\mathbf{A}_S = \frac{1}{\sqrt{S}}[\mathbf{x}_n^{1} \cdots \mathbf{x}_n^{S}]^\intercal \in \R^{S \times M}$.
Because  the $\mathbf{x}_n^{(s)}$ are i.i.d. realizations of a whitened Gaussian distribution whose covariance matrix is the identity, we obtain from~\cite[Corollary 5.35]{vershynin_2012} that, for any $t>0$
\begin{equation}
    \mbox{Pr} \left( 1 - \delta_S \leq  \sigma_{\min} (\mathbf{A}_S) \leq  \sigma_{\max} (\mathbf{A}_S) \leq 1 + \delta_S  \right) \geq 1 - 2e^{ -t^2/2} , \label{eq:Proof_pgcmrate-3} 
\end{equation}
where $\delta_S = \frac{\sqrt{M}+t }{\sqrt{S}}$ and $\sigma_{\max} (\mathbf{A}_S)$ (resp. $\sigma_{\min} (\mathbf{A}_S)$) denotes the largest (resp. smallest)  singular value of $\mathbf{A}_S$. Using~\cite[Lemma 5.36]{vershynin_2012}, we get from \eqref{eq:Proof_pgcmrate-3} that 
\begin{equation}
    \mbox{Pr} \left( \|\mathbf{A}_S^\intercal \mathbf{A}_S - \mathbf{I}\| \leq h_S \right) \geq 1 - 2e^{ -t^2/2}, \label{eq:Proof_pgcmrate-4} 
\end{equation}
with $h_S := 3 \max(\delta_S,\delta_S^2)$.
% Proposition~\ref{prop:jdcond} and Assumption~\ref{assumpBound} we have $\| \bm{\Sigma}_n^{-1}   \|  \leq \frac{1}{ \epsilon_0 }$ and thus, with
Then, from~\eqref{eq:Proof_pgcmrate-2},
\begin{equation}
    \| \bm{\Sigma}_n^{-1/2} \bm{\Sigma}_{n,S} \bm{\Sigma}_n^{-1/2} - \mathbf{I} \|  =   \| \mathbf{A}_S^\intercal \mathbf{A}_S - \mathbf{I} \| .\label{eq:Proof_pgcmrate-5}  % +  \frac{\epsilon_S }{ \epsilon_0 } 
\end{equation}
Combining the latter with~\eqref{eq:Proof_pgcmrate-4}, we obtain
\begin{multline}
    \mbox{Pr} \left(  \| \bm{\Sigma}_n^{-1/2} \bm{\Sigma}_{n,S} \bm{\Sigma}_n^{-1/2} - \mathbf{I} \| \leq h_S  \right) \geq 1 - 2e^{ -t^2/2}. \label{eq:Proof_pgcmrate-6}  % + \frac{\epsilon_S }{ \epsilon_0 } 
\end{multline}

It follows, using the fact that the vectors $   \{  \mathbf{x}_n^{(s)}   \}_{s=1}^S$ are independent of the vectors $  \{  \mathbf{x}_{n'}^{(s)}   \}_{s=1}^S$ for $n \neq n'$, that
\begin{align}
    \mbox{Pr} \left( \TS \leq h_S   \right) &\geq \prod_{n=1}^N \mbox{Pr}  \left( \| \bm{\Sigma}_n^{-1/2} \bm{\Sigma}_{n,S} \bm{\Sigma}_n^{-1/2} - \mathbf{I} \| \leq h_S  \right) \notag \\
            &\geq ( 1 - 2e^{ -t^2/2} )^N ,  \label{eq:Proof_pgcmrate-7} 
\end{align}
% \begin{align}
%     \mbox{Pr} &\left( \TS \leq h_S   \right)  \notag  \\
%             & \geq \prod_{n=1}^N \mbox{Pr}  \left( \| \bm{\Sigma}_n^{-1/2} \bm{\Sigma}_{n,S} \bm{\Sigma}_n^{-1/2} - \mathbf{I} \| \leq h_S  \right) \notag \\
%             &\geq ( 1 - 2e^{ -t^2/2} )^N ,  \label{eq:Proof_pgcmrate-7} 
% \end{align}
which proves that, under GCM, $\TS  \overset{p}{\to} 0 \; \text{ as } \; S \to \infty$.

Finally, for $t>0$ and $S$ large enough such that 
$h_S    = 3\delta_S    < 1$, we deduce from~\eqref{eq:Proof_pgcmrate-7} and Lemma~\ref{lem:Bound_Qs_with_Ts} that
\begin{align}
    \mbox{Pr} &  \left( \max_{\aPhi \in \St(M)}  Q_S (\aPhi) \leq MN \frac{h_S^2} { 1 - h_S } \right) \\
    & \geq \mbox{Pr} \left( \TS \leq h_S  \right)   \geq  ( 1 - 2e^{ -t^2/2} )^N.
\end{align}

\section{Proofs of Propositions~\ref{prop:QN_dir_TLNMF} and~\ref{prop:QN_dir_JD}}

\subsection{Optimality Conditions of Problem~\eqref{eq:QN_dir}}

In order to prove Propositions~\ref{prop:QN_dir_TLNMF} and~\ref{prop:QN_dir_JD}, we first explicit the optimality conditions of the quadratic problem~\eqref{eq:QN_dir}.

\begin{lemm}\label{lemm:opt_cond}
Let $\domain$ be the set of anti-symmetric matrices and $\tilde{\LH}$ be positive semi-definite. Then, $\LE \in \domain$ is solution of problem~\eqref{eq:QN_dir} if and only if  $\LE= (\LM - \LM^\tT)/2$ with $\LM \in \R^{M\times M}$ solution of
\begin{equation}\label{eq:lemm_opt_cond}
 \frac14 \sum_{cd} [\mathcal{Z}]_{abcd} \left[\frac{\LM - \LM^\tT}{2} \right]_{cd} = -  [\LG^{\anti}]_{ab}
\end{equation}
where $[\mathcal{Z}]_{abcd} = [ \tilde{\LH}]_{abcd} + [ \tilde{\LH}]_{cdab}+ [ \tilde{\LH}]_{badc}+ [ \tilde{\LH}]_{dcba}$. 
\end{lemm}
\begin{proof}

Using the parametrization of the anti-symmetric matrix $\LE= (\LM - \LM^\tT)/2$, solving the constrained problem~\eqref{eq:QN_dir} is equivalent to solving the unconstrained problem
\begin{equation}\label{eq:minM}
    \widehat{\LM} =  \argmin_{\LM \in \R^{M \times M}}\; J(\LM)
\end{equation}
with 
\begin{equation*}
    J(\LM) = \left\langle \LG , \frac{ \LM - \LM^\tT}{2}  \right\rangle +  \frac{1}{2} \left\langle \frac{ \LM - \LM^\tT}{2} | \tilde{\LH}   |\frac{ \LM - \LM^\tT}{2} \right\rangle.
\end{equation*}
where we recall that  $\left\langle \cdot, \cdot \right\rangle$ and $\left\langle \cdot | \cdot | \cdot \right\rangle$ are defined after equation~\eqref{eq:TaylorExp}.
Given that $J$ is convex (as $\tilde{\LH}$ is positive semi-definite), we get from the first-order optimality conditions that $\nabla J(\widehat{\LM}) = \mathbf{0}$.

Because $ \lb \LG , \LM^\tT \rb = \lb \LG^\tT ,  \LM \rb$, the first term  of $J$ is equal to $\lb  \LG^{\anti}  , \LM \rb$ and its gradient is thus given by $\LG^{\anti}$. To compute the gradient of the second term, let us first expand
\begin{align}
    % \lb \LM + \bm{\Delta} |  \tilde{\LH} | \LM  + \bm{\Delta} \rb & =  \lb \LM  |  \tilde{\LH} | \LM  \rb + o(\|\bm{\Delta}\|^2) \notag \\
\lb \LM + \bm{\Delta} |  \tilde{\LH} | \LM  + \bm{\Delta} \rb & =  \lb \LM  |  \tilde{\LH} | \LM  \rb + o(\|\bm{\Delta}\|) \notag \\
    & \quad +  \lb \LM |  \tilde{\LH} |  \bm{\Delta} \rb + \lb \bm{\Delta} |  \tilde{\LH} | \LM  \rb. \label{eq:Proof_QN_dir}
\end{align}
Then, by definition of the inner product, we get
\begin{align*}
   & \lb \bm{\Delta}  |  \tilde{\LH} | \LM  \rb  = \sum_{ab} [\bm{\Delta}]_{ab} \sum_{cd}  [ \tilde{\LH}]_{abcd} [\LM]_{cd}, \\
   &  \lb \LM |  \tilde{\LH} |  \bm{\Delta} \rb = \sum_{ab} [\bm{\Delta}]_{ab} \sum_{cd}  [ \tilde{\LH}]_{cdab} [\LM]_{cd} .
\end{align*}
Similar expressions can be obtained when transposing one or both variables of the bilinear form in~\eqref{eq:Proof_QN_dir}.
We finally deduce from these expansions that
\begin{align*}
   & \left[ \nabla \lb \cdot |  \tilde{\LH} | \cdot \rb (\LM) \right]_{ab} = \sum_{cd} ([ \tilde{\LH}]_{abcd} + [ \tilde{\LH}]_{cdab} ) [\LM]_{cd} \\
   & \left[\nabla \lb (\cdot)^\tT |  \tilde{\LH} | \cdot \rb (\LM)\right]_{ab} =  \sum_{cd} ([ \tilde{\LH}]_{cdab} + [ \tilde{\LH}]_{badc} ) [\LM^\tT]_{cd}\\
   &\left[ \nabla \lb \cdot |  \tilde{\LH} |  (\cdot)^\tT \rb (\LM)\right]_{ab} = \sum_{cd} ([ \tilde{\LH}]_{abcd} + [ \tilde{\LH}]_{dcba} ) [\LM^\tT]_{cd}\\
   & \left[\nabla \lb  (\cdot)^\tT |  \tilde{\LH} |  (\cdot)^\tT \rb (\LM) \right]_{ab}= \sum_{cd} ([ \tilde{\LH}]_{badc} + [ \tilde{\LH}]_{dcba} ) [\LM]_{cd}
\end{align*}
Combining all these expressions, we obtain that  $\nabla J(\widehat{\LM}) = \mathbf{0}$ is equivalent to~\eqref{eq:lemm_opt_cond}, which completes the proof.
\end{proof}

\subsection{Proof of Proposition~\ref{prop:QN_dir_TLNMF}}\label{apdx:proof_prop_QN_dir_TLNMF}

By injecting the expression of $ [\tilde{\mathbf{ \LH }}]_{abcd} =  \delta_{ac} \delta_{bd} [\bm{\Gamma}]_{ab} $ into~\eqref{eq:lemm_opt_cond}
%the generic optimality condition of Lemma~\ref{lemm:opt_cond},
we get, for all $(a,b)$,
\begin{equation}\label{eq:Opt_Cond_TL}
       [\bm{\Gamma}^{\symm}]_{ab} \left[\frac{\LM - \LM^\tT}{2} \right]_{ab} = -  [\LG^{\anti}]_{ab}.
\end{equation}

If, for a couple $(a,b)$, we have  $[\bfGamma^{\symm}]_{ab}  =  0$, we get from the definition of  $\bfGamma$ in~\eqref{eq:Gamma_CS} that $\forall (n,s), [\X^{(s)}]_{bn}=0$. It  follows from \eqref{eq:grad_CS} that $ [\LG]_{ab}  = 0$ 
and thus that $[\LG^{\anti}]_{ab}  = 0 $. Hence, for all $(a,b)$ such that  $[\bfGamma^{\symm}]_{ab}  =  0$, fixing $\left[\frac{\LM - \LM^\tT}{2} \right]_{ab}$ to any real is solution of~\eqref{eq:Opt_Cond_TL}. In this work we choose the value $0$ (second line in~\eqref{eq:prop_QN_dir_TL}). For all $(a,b)$ such that $[\bfGamma^{\symm}]_{ab}  \neq  0$, we get from~\eqref{eq:Opt_Cond_TL} the expression given in the first line in~\eqref{eq:prop_QN_dir_TL}.

Finally, let us  show that the optimal solution is a descent-direction, i.e.,  $\lb \LG , \LE \rb  < 0$. 
From the definition of $\bfGamma$, 
we get that $[\bfGamma^{\symm}]_{ab} \geq 0 $.
Moreover, for all $(a,b)$ such that $[\bfGamma^{\symm}]_{ab}  =  0$ we fixed $[\LE]_{ab}=0$. Therefore,
\begin{equation}
        \lb \LG , \LE \rb   = \lb \LG^{\anti} , \LE \rb  = 
    -  \sum_{\substack{(a,b) \\ [\LE]_{ab} \neq 0}} \frac{ [\LG^{\anti}]_{ab}^2 } { \bfGamma^{\symm}_{ab}} < 0. 
\end{equation}
%which completes the proof.

\subsection{Proof of Proposition~\ref{prop:QN_dir_JD}}\label{apdx:proof_prop_QN_dir_JD}

Again, by injecting the expression of
\begin{equation}
	[\tilde{\LH} ]_{abcd} = \delta_{ac}\delta_{bd} 
	[\bfGamma]_{ab} + \delta_{ad} \delta_{bc} - 2 \delta_{abcd} . 
\label{eq:jdhessapprox}
\end{equation}
into~\eqref{eq:lemm_opt_cond}
%the generic optimality condition of Lemma~\ref{lemm:opt_cond},
we obtain, for all $(a, b)$,
\begin{equation}\label{eq:Opt_Cond_JD}
        ([\bm{\Gamma}^{\symm}]_{ab} -1) \left[\frac{\LM - \LM^\tT}{2} \right]_{ab} = -  [\LG^{\anti}]_{ab}.
\end{equation}

If, for a couple $(a,b)$, we have  $[\bfGamma^{\symm}]_{ab}  =  1$, we get from the definition of  $\bfGamma$ in~\eqref{eq:Gamma} that $\forall n $, $\frac{  [ \aPhi ( \Sigma_{n,S} + \eps \Id) \aPhi^\intercal ]_{bb}  }
{ [ \aPhi (\Sigma_{n,S} + \eps \Id) \aPhi^\intercal ]_{aa} }   = 1 $. It  follows from \eqref{eq:grad_LS} that $ [\LG]_{ab}  = [\LG]_{ba} $ 
and thus that $[\LG^{\anti}]_{ab}  = 0 $. Hence, for all $(a,b)$ such that  $[\bfGamma^{\symm}]_{ab}  =  1$, fixing $\left[\frac{\LM - \LM^\tT}{2} \right]_{ab}$ to any real number is solution of~\eqref{eq:Opt_Cond_JD}. In this work we choose the value $0$ (second line in~\eqref{eq:prop_QN_dir_JD}). For all $(a,b)$ such that $[\bfGamma^{\symm}]_{ab}  \neq  1$, we get from~\eqref{eq:Opt_Cond_JD} the expression given in the first line in~\eqref{eq:prop_QN_dir_JD}.

Finally, let us  show that the optimal solution is a descent-direction, i.e.,  $\lb \LG , \LE \rb  < 0$. 
From the definition of $\bfGamma$, 
we get that $[\bfGamma^{\symm}]_{ab} \geq 1 $.
Moreover, for all $(a,b)$ such that $[\bfGamma^{\symm}]_{ab}  =  1$ we fixed $[\LE]_{ab}=0$. Therefore,
\begin{equation}
        \lb \LG , \LE \rb   = \lb \LG^{\anti} , \LE \rb  = 
    -  \sum_{\substack{(a,b) \\ [\LE]_{ab} \neq 0}} \frac{ [\LG^{\anti}]_{ab}^2 } { \bfGamma^{\symm}_{ab} - 1 } < 0. 
\end{equation}
%which completes the proof.

\section{Multi-Initialization Strategy}\label{apdx:init_strategy}

% As the objective functions $C_S$, $L_S$ and $I_S$ are non-convex, 

We describe a  multi-initialization strategy that allows to obtain numerical solutions that are close to the solution sets $\Omega^*$ and $\dotr{\Omega}$. To that end, we first consider $P$ random initializations
\begin{equation}
    \{ (\aPhi_0^{(p)},\W_0^{(p)},\Ht_0^{(p)}) \}_{p \leq P}
%    \{ (\aPhi^0_{(p)},\W^0_{(p)},\Ht^0_{(p)}) \}_{p \leq P}
\end{equation}
from which we run the TL-NMF and JD+NMF solvers in Algorithms~\ref{alg:tl-nmf} and~\ref{alg:jd+nmf}. Given an initialization, each method is run for $J$ iterations. In order to robustify the process, the solutions obtained with JD+NMF (resp. TL+NMF) 
$ \{ (\aPhi_J^{(p)},\W_J^{(p)},\Ht_J^{(p)}) \}_{p \leq P}$
are used as extra $P$ initializations
\begin{equation}
    \{ (\aPhi_0^{(p)},\W_0^{(p)},\Ht_0^{(p)})  \}_{P+ 1 \leq  p \leq 2P}
\end{equation}
for TL-NMF (resp. JD+NMF). 
Finally, for TL-NMF, we preserve only the solution $(\haPhi^\ast,\hW^\ast,\hHt^\ast)$ that achieves the smallest $C_S$. %  after $J$ iterations. 
For JD+NMF, we preserve only the solution $\dotr{\haPhi}$ that achieves the smallest $L_S$. We then minimize $I_S(\dotr{\haPhi}, \cdot, \cdot)$, started from $ \{ (\W_0^{(p)},\Ht_0^{(p)} ) \}_{p \leq 2P}$, and preserve the best solution $ ( \dotr{\hW} ,\dotr{\hHt} )$ that achieves the smallest loss $I_S$.

\ifCLASSOPTIONcaptionsoff
  \newpage
\fi

% \newpage
% \input{special.tex}

\bibliographystyle{IEEEtran}
\bibliography{IEEEabrv,refs}

\end{document}